\documentclass{article}

\PassOptionsToPackage{numbers, compress}{natbib}

\usepackage[preprint]{neurips_2026}

\usepackage[utf8]{inputenc} 
\usepackage[T1]{fontenc}    
\usepackage{url}            
\usepackage{booktabs}       
\usepackage{makecell}       
\usepackage{multirow}       
\usepackage{arydshln}       
\usepackage{algorithm}      
\usepackage{algpseudocode}  
\usepackage{graphicx}       
\usepackage{wrapfig}        
\usepackage{amsmath}        
\usepackage{amsfonts}       
\usepackage{amsthm}
\usepackage{nicefrac}       
\usepackage{microtype}      
\usepackage[table]{xcolor}  
\usepackage{listings}       
\usepackage[most]{tcolorbox}
\usepackage{hyperref}       

\definecolor{tableheader}{HTML}{F4F8FC}
\definecolor{tablesubheader}{HTML}{FBFCFE}
\definecolor{tablesection}{HTML}{EEF2F5}
\definecolor{tablehighlight}{HTML}{CFE8FF}
\definecolor{tableblue}{HTML}{1E5D96}
\definecolor{algobox}{HTML}{F3F5F7}
\definecolor{algoframe}{HTML}{D8DEE6}
\definecolor{toprankone}{HTML}{7DB8FF}
\definecolor{topranktwo}{HTML}{B7DAFF}
\definecolor{toprankthree}{HTML}{E0F0FF}
\definecolor{deltagreen}{HTML}{18864B}
\definecolor{deltared}{HTML}{B42318}
\definecolor{deltagray}{HTML}{667085}
\definecolor{promptbg}{HTML}{F3FAFF}
\definecolor{promptframe}{HTML}{9DC7E8}
\definecolor{prompttitlebg}{HTML}{E1F0FB}
\newcommand{\rankbox}[2]{\begingroup\setlength{\fboxsep}{1.1pt}\colorbox{#1}{\strut #2}\endgroup}
\newcommand{\rankone}[1]{\rankbox{toprankone}{\textbf{#1}}}
\newcommand{\ranktwo}[1]{\rankbox{topranktwo}{\textbf{#1}}}

\newcommand{\ourscell}[1]{\cellcolor{tablehighlight}\textbf{#1}}
\newcommand{\oursval}[1]{\cellcolor{tablehighlight}#1}

\newcommand{\methodref}[2]{#1~{\scriptsize\citep{#2}}}
\newcommand{\tablesectionrow}[2]{\rowcolor{tablesection}\multicolumn{#1}{:l:}{\textbf{\textcolor{tableblue}{#2}}}\\\hdashline}
\newcommand{\ablgood}[1]{\textcolor{deltagreen}{#1}}
\newcommand{\ablbad}[1]{\textcolor{deltared}{#1}}
\newcommand{\ablzero}[1]{\textcolor{deltagray}{#1}}
\newcommand{\ablcell}[2]{\begin{tabular}[c]{@{}c@{}}#1\\[-2pt]{\tiny #2}\end{tabular}}
\newtheorem{definition}{Definition}
\newtheorem{lemma}{Lemma}
\newtheorem{theorem}{Theorem}
\newtheorem{corollary}{Corollary}
\lstdefinestyle{promptstyle}{
  basicstyle=\ttfamily\scriptsize,
  columns=fullflexible,
  breaklines=true,
  breakatwhitespace=false,
  keepspaces=true,
  showstringspaces=false
}
\newtcblisting{promptbox}[1]{
  enhanced,
  breakable,
  listing only,
  colback=promptbg,
  colframe=promptframe,
  boxrule=0.55pt,
  arc=2pt,
  left=5pt,
  right=5pt,
  top=3pt,
  bottom=3pt,
  title={#1},
  coltitle=tableblue,
  fonttitle=\bfseries\small,
  boxed title style={colback=prompttitlebg,colframe=promptframe,boxrule=0.45pt,arc=2pt},
  attach boxed title to top left={xshift=5pt,yshift=-1.5mm},
  listing options={style=promptstyle}
}

\title{TRACER: Turn-level Regret Matching with Inner Reinforcement Credit for Cooperative Multi-LLM Reasoning}

\author{
  Chusen Li\textsuperscript{1,2,*} \quad
  Zhou Liu\textsuperscript{3,*}\\
  Shuigeng Zhou\textsuperscript{1} \quad
  Wentao Zhang\textsuperscript{3}\\
  \normalfont\small
  \textsuperscript{1}Fudan University \quad
  \textsuperscript{2}Zhongguancun Academy\\
  \normalfont\small
  \textsuperscript{3}Academy for Advanced Interdisciplinary Studies, Peking University\\
  \normalfont\small
  \textsuperscript{*}Equal contribution\\
  \normalfont\small
  \texttt{lichusen@whu.edu.cn} \quad
  \texttt{zhouliu25@stu.pku.edu.cn}\\
  \normalfont\small
  \texttt{sgzhou@fudan.edu.cn} \quad
  \texttt{wentao.zhang@pku.edu.cn}
}

\begin{document}

\maketitle

\begin{abstract}
Large language models increasingly rely on either reinforcement learning or multi-agent prompting to improve reasoning, yet these two paradigms remain difficult to combine. Directly applying single-agent reinforcement learning to multi-turn multi-agent systems faces following dilemmas: \textbf{i)Sparse rewards, role-level free-riding and excessive training overhead.} \textbf{ii)Agents only imitate to collaborate.} \textbf{iii) Fixed collaboration protocol falls into oscillating local optimum.} We introduce \textbf{TRACER}, a turn-level reinforcement framework for cooperative multi-LLM reasoning. TRACER separates collaborative decision making into a controller-regret layer, where controllers learn whether the agents should speak or skip the current round through regret matching, and a generation-credit layer, which optimizes proposer and reviewer utterances with role-specific GSPO rewards. This design \textbf{i)assigns credit at the level of both action modes and generated utterances, thus avoiding free-riding and sparse rewards. We only expand the choices made by the controllers, thus greatly reducing computational cost of training.} Moreover, \textbf{ii)agents acquire collaborative capability as they learn when to utter and what to speak.} Finally, \textbf{iii)by designing binary actions ingeniously, we extend classical game theory established for finite action spaces to deep learning, thus achieving mathematically rigorous convergence.} We train all local RL-style methods on the GSM8K training split and evaluate on held-out GSM8K, MATH500, and GPQA-Diamond to measure in-domain accuracy, cross-benchmark generalization, inference cost, and correction-preservation behavior. The resulting framework provides a compact and reproducible testbed for studying learned collaboration policies beyond fixed debate, voting, or aggregation protocols. Code is available in \url{https://github.com/Shark-Forest/TRACER}.
\end{abstract}

\begin{figure}[t]
\centering
\includegraphics[width=\columnwidth]{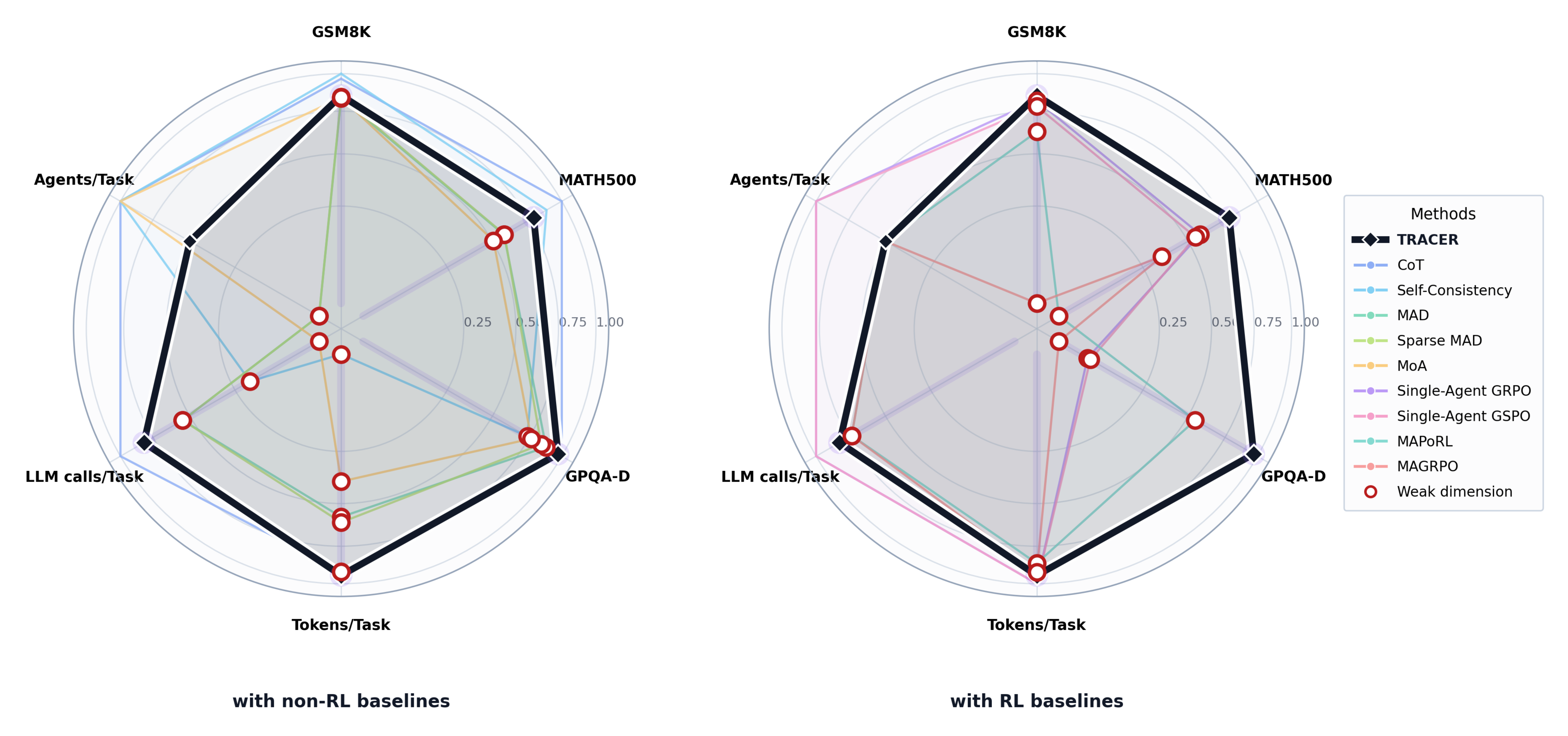}
\caption{\textbf{Radar comparison of TRACER against non-RL and RL baselines based on Qwen2.5-7B-Instruct across accuracy and efficiency metrics}. GSM8K, MATH500, and GPQA-D are accuracy metrics where larger values are better, while Tokens/Task, LLM calls/Task, and Agents/Task are cost metrics where smaller values are better and are therefore inverted for visualization. All axes are normalized so that points farther from the center indicate better performance. Hollow red circles mark weak dimensions of baseline methods, highlighting where competing approaches incur accuracy drops or higher inference cost. TRACER maintains a more balanced profile across tasks, i.e., preserving non-trivial reasoning accuracy and achieving multi-agent efficiency.
}
\label{fig:tracer_radar}
\end{figure}

\section{Introduction}

LLMs have demonstrated outstanding capabilities across both general-purpose and specialized task domains. In recent years, two mainstream paradigms have been widely adopted to further enhance LLM performance on challenging tasks: reinforcement learning (RL) and multi-agent systems (MAS)~\citep{liu2025reinforcement}. RL treats each utterance generated by an LLM as an action and optimizes model parameters on policy in multi-turn settings, as in PPO~\citep{schulman2017proximal}, GRPO~\citep{feng2025group}, GSPO~\citep{zheng2025gspo}, and related methods. Through RL, pretrained LLMs can better align with human preferences and task objectives, thereby improving their decision-making and reasoning capabilities on complex tasks~\citep{ouyang2022training}. MAS, by contrast, regards each LLM as an intelligent agent and designs role-specific prompts, multi-round interaction mechanisms, and collaborative workflows that guide agents to divide labor, negotiate, and jointly solve complex tasks~\citep{li2024survey}.

Considering the advantages of these two paradigms, a natural idea is to integrate them and train multi-agent systems (MAS) using reinforcement learning (RL)-based methods. However, directly transferring single-agent RL training paradigms to MAS encounters substantial difficulties~\citep{gronauer2022multi}.

\textbf{i) Turn-level reward sparsity~\citep{bui2025preference}, free-riding~\citep{liu2023lazy}, and excessive training overhead~\citep{huh2023multi}.}
Directly applying large-model RL algorithms (e.g., GRPO) to multi-round collaborative training leads to reward sparsity across intermediate rounds and reward confounding among different roles. In a reviewer–proposer scenario, a typical case occurs when the reviewer makes an incorrect judgment while the proposer ultimately retains the correct answer; under these circumstances, the reviewer still receives a positive reward. In contrast, if every round requires sampling and all candidate solutions are unfolded, the training cost will increase exponentially with the number of rounds~\citep{hu2024valuebaseddeepmultiagentreinforcement}.

\textbf{ii) Agents merely imitate collaboration.}
If multiple agents are trained separately and independently and are only distinguished by prompts at inference time, they may merely imitate a specific role without acquiring genuine collaborative capabilities~\citep{han2024llm}. In some cases, they only consume rounds without fostering meaningful discussion and may even corrupt a previously correct answer.

\textbf{iii) Current RL and MAS-RL methods exhibit training instability and low convergence rates.}
There has been some exploration of collaboratively training MAS with RL methods. However, these approaches often introduce prior preferences for specific collaboration schemes such as polling, debate, and voting, thus making them prone to becoming trapped in suboptimal policies or oscillating repeatedly among them. Their convergence speed is slow, the training process is unstable, and they often converge only to unstable local optima~\citep{liu2023maximum}. As for single-agent RL, unstable training curves result in poor generalization performance, and additional fine-tuning is often required when the models are transferred to datasets that differ from the training set.

To address the above issues, we propose \textbf{TRACER} (\textbf{T}urn-level \textbf{R}egret m\textbf{A}tching with inner reinforcement \textbf{ C}r\textbf{E}dit for collaborative \textbf{R}easoning). TRACER is a two-layer framework that jointly learns when the agents should speak and how each utterance should be optimized. Concretely, our contributions are as follows:

\begin{itemize}
     \item We propose a two-layer policy system with clearly defined responsibilities. The \textbf{Controller-Regret Layer} consists of two
        controllers that decide whether each agent should speak, while the \textbf{Generation-Credit Layer} optimizes the utterances produced by the agents. By design, the system adapts to the current context and encourages productive discussion, which helps reduce invalid loops and unnecessary backtracking.
        \item In the \textbf{Controller-Regret Layer}, we define \textbf{counterfactual regret} and \textbf{cumulative counterfactual regret} for the two actions available to each controller. By adopting a \textbf{regret-matching} algorithm to update each controller's policy, the system learns to act based on the current context. As a result, answers with higher latent confidence are more likely to be retained, whereas answers with lower latent confidence are more likely to be updated.

   \item In the \textbf{Generation-Credit Layer}, we explicitly evaluate the contribution of each
  agent. For the reviewer, rewards are assigned according to the correctness of its judgment,
  whereas for the proposer, rewards are assigned according to the correctness of its answer. By
  designing step-wise, role-specific rewards, we mitigate the reward sparsity and role-reward
  aliasing issues that arise when reinforcement learning is directly applied to multi-round
  collaboration. Moreover, because the actual training trajectories are generated through single-step sampling from the controllers rather than standard reinforcement learning rollouts, our approach maintains low training cost.
\end{itemize}

\section{Related Work}

\subsection{Reinforcement Learning}

Reinforcement Learning (RL) is a fundamental paradigm for sequential decision-making, formally modeled as a Markov Decision Process~\citep{sutton1998reinforcement}. Built on the Policy Gradient Theorem, classic policy gradient algorithms (e.g., Actor-Critic \citep{sutton1999policy}) have demonstrated strong performance in deep learning. TRPO~\citep{schulman2015trust} and its optimized variant PPO further resolve the limitations of adaptive step tuning and high computational cost in prior methods, becoming the de facto standard for large language model (LLM) training~\citep{ouyang2022training}, and spawning a series of LLM-specific reasoning enhancement algorithms. Another seminal related framework is Counterfactual Regret Minimization (CFR), designed for imperfect-information adversarial games~\citep{zinkevich2007regret}, which introduces counterfactual value calculation to eliminate state reachability uncertainty. However, most existing CFR-based methods are applied for competitive settings and finite action spaces, making them poorly suited for fully observable, cooperative and information-perfect multi-LLM tasks whose action space is infinite~\citep{hartley2017multi}, where credit assignment ambiguity, free-riding problems, and the lack of convergence guarantees for collaborative strategies remain critical unsolved challenges.

\subsection{Multi-Agent System}

LLM-based Multi-Agent System (MAS) has emerged as a powerful paradigm to tackle complex reasoning and decision-making tasks beyond the capability of a single LLM, by decomposing workloads and distributing subtasks across multiple collaborative LLM agents~\citep{tran2025multi}. Benefiting from the strong general reasoning and generalization ability of pre-trained LLMs, these systems achieve excellent performance in diverse downstream tasks without extensive domain-specific fine-tuning~\citep{li2023camel}. However, prevailing MAS predominantly rely on hand-crafted collaboration rules and fixed coordination mechanisms, rather than adaptive learning-based strategies~\citep{qian2024chatdev}~\citep{hong2023metagpt}. This not only causes rigid behavior imitation and reasoning deviation in multi-round interactions, but also leads to fundamental credit assignment ambiguity when mapping overall team performance to individual agents. Worse still, most existing frameworks are empirical, lacking rigorous theoretical convergence guarantees for their collaborative policies. To address these gaps, this work proposes a turn-level counterfactual regret formulation adapted for cooperative multi-LLM scenarios, which enables precise and fair credit assignment via single-step counterfactual action replacement, while providing provable convergence guarantees for learned strategies~\citep{foerster2018counterfactual}.

\raggedbottom
\section{Problem Definition}

We consider a setting in which multiple controllers engage in a collaborative game. Each controller has its own utility function and decides whether its associated agent should speak. In this paper, we focus on the case of two controllers and two agents, referred to as the proposer and the reviewer. Both agents are pretrained LLM-based agents that follow stochastic policies $\pi_{\theta_i}$. Given a prompt $x \sim \mathcal{D}$, $\pi_{\theta_i}(y \mid x)$ denotes the conditional probability of generating a response $y$, where $x$ consists of the raw question, a role-specific task description, and concise, robust structured information tailored differently to the proposer and the reviewer. In the first turn, the proposer is required to speak. From the second turn onward, the reviewer’s controller and the proposer’s controller alternate in deciding whether their respective agent should speak or skip the round. The round horizon is limited to $T$.

\section{TRACER}

\subsection{Two-Layer Controller Collaboration Architecture}

 We first describe how the final answer is obtained. When a new answer is proposed, its initial vote is set to zero. Each time the reviewer provides a positive judgment, the vote of the pending answer increases by one. If the reviewer provides a negative judgment, the vote decreases by one. If the reviewer’s judgment is invalid or its controller skips the round, the vote remains unchanged. The vote of the pending answer, together with the iteration index, determines a discrete state $s^t$, which is encoded as structured information and incorporated into the context. Whenever a new answer is generated, it replaces the previous one. If the new answer is identical to the previous answer, however, its vote is not reset. The final answer is defined as the latest proposed answer.

 \begin{figure}[t]
\centering
\includegraphics[width=\columnwidth]{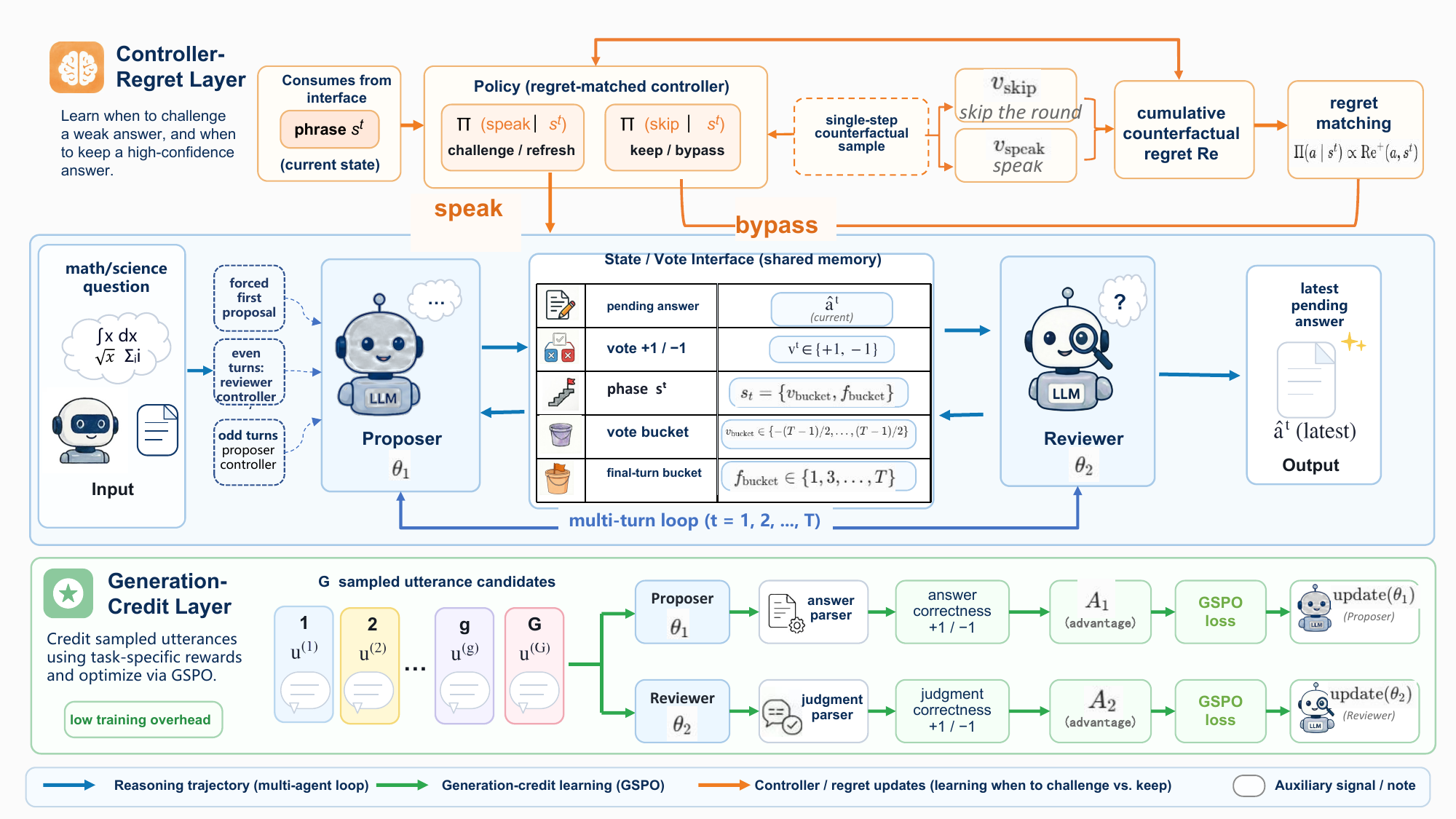}
\caption{\textbf{TRACER architecture.} The controller-regret layer learns when to skip or speak through regret matching, while the generation-credit layer assigns role-specific rewards to proposer and reviewer utterances and optimizes them with GSPO. The state/vote interface serves as shared memory across the multi-turn reasoning loop.}
\label{fig:tracer_architecture}
\end{figure}

Figure~\ref{fig:tracer_architecture} shows how the TRACER works in detail. TRACER decomposes cooperative reasoning into two named layers. The \textbf{Controller-Regret Layer} is a behavioral framework for the Multi-Agent System. Concretely, we train two outer policies $\Pi_1$ and $\Pi_2$. The proposer is forced to utter in the first turn. Then given a phase $s^t$ in the $t$-th turn when $t \geq 2$, $\Pi_i(\textbf{speak} \mid s^t)$ denotes the conditional probability of speaking and $\Pi_i(\textbf{skip} \mid s^t)$ denotes the conditional probability of keeping silent, i.e. skipping the round. There are $\frac{t(t + 1)}{2}$ different phases, as there are $\frac{t + 1}{2}$ different values of the iteration and there are $t$ different values of the vote from $-\frac{t - 1}{2}$ to $\frac{t - 1}{2}$. The reviewer's controller, i.e. $\Pi_2$, acts in the even rounds and the proposer's controller, i.e. $\Pi_1$ acts in the odd rounds. The Controller-Regret Layer decides the action and helps the system keep the high-confidence answer while adaptively replace the low-confidence answer, which induces the system to get closer to the correct answer. The scenario is equivalent tothe \textbf{Perfect Information Game} and we adopt
\textbf{Regret-Matching} to optimize it. Theoretical details are available in the appendix. By design, only the action sampled by the controller enters the training trajectory, thus we can apply step-wise GSPO reward and avoid excessive training overhead. What's more, the system will perceive its own capability and confidence of the pending answer after training, and both the controllers tend to skip the current round when they are uncertain about their own opinions in the reasoning process. The proposer has a lower probability of ruining the correct answer during revision, and the reviewer is also less likely to offer incorrect feedback.

The \textbf{Generation-Credit Layer} is the inner generation optimization layer. We use \textbf{Role-Specific Rewards} and \textbf{Group Sequence Policy Optimization, GSPO} to targetedly optimize the LLM parameters and boost the abilities of the reviewer and the proposer. Together, two layers guide the MAS to act adaptively and foster discussion.

\subsection{Credit and Regret Matching}

\textbf{Generation-Credit Layer: Intra-Group Advantage for utterance optimization.}\quad For the agent, if it's invoked at the $t$-th turn, a dedicated parser will extract structured information from the utterance. If the agent is the proposer, the reward $r_1$ only depends on the answer and is not related to the reasoning text. If the answer is correct, $r_1$ = $+1$. Otherwise $r_1$ = $-1$. If the agent is the reviewer, the reward $r_2$ only depends on the judgement and is not related to the quality of the reasoning text, and $r_2 = +1$ for a correct judgement, $r_2 = -1$ for a wrong judgement, while $r_2 = 0$ for an invalid judgement.

At the $t$-th turn, suppose $t \geq 2$ and an agent is invoked. We sample $G$ responses: $\left\{ y_{t}^{(j)} \right\}_{j=1}^{G}$. $r_i^{(j)}$ denotes the value attributed to $y_t^{(j)}$ and the character $i$ ($i \in \{1,2\}$), where $i = 1$ denotes the proposer and $i = 2$ denotes the reviewer. Then we compute $\mu_i$ which denotes the mean of $\{r_i^{(j)}\}_{j=1}^G$ and $\sigma_i$ which denotes the standard deviation of $\{r_i^{(j)}\}_{j=1}^G$. Then:
\[
    A_i^{(j)} = \frac{r_i^{(j)} - \mu_i}{\sigma_i + \epsilon}, \qquad
    \mathcal{L}_{\text{GSPO}}(\theta_{i}) = -\mathbb{E}\left[\min\left(\rho_{i,j} A_i^{(j)},\ \mathrm{clip}\left(\rho_{i,j}, 1-\epsilon', 1+\epsilon'\right) A_i^{(j)}\right)\right].
\]
where $\theta_{i}$ denotes the parameter group of the agent $u_i$ and
\[
    \rho_{i,j} = \left( \frac{\pi_{\theta_{i}}(y_t^{(j)} \mid \mathcal{C}_t)}{\pi_{\theta_{i}^{old}}(y_t^{(j)} \mid \mathcal{C}_t)}\right )^{\frac{1}{|y_t^{(j)}|}}.
\]
Here $\mathcal{C}_t$ is the prompt collection at the $t$-th turn. The details of how we build $\mathcal{C}_t$ are available in the appendix.

The benefit of this set of definitions is that we can clearly measure the contribution of a specific role aligned with the training goal without introducing extra noise. We train the reviewer to be a discriminator and the proposer to be a generator. The reason text only serves as the raw material for the two parsers (a proposer parser and a reviewer parser) to assemble a structured prompt collection and the phase $s^t$, rather than entering the context collection directly. In this way, we completely separate policy updating from context construction, instead of conflating them. This is particularly important in scenarios where small models are applied to Multi-Agent Systems.

\textbf{Controller-Regret Layer: regret matching for action-mode selection.}\quad The two controllers aim to determine to invoke the agents or not at the $t$-th turn, where $t \geq 2$. We introduce \textbf{counterfactual regret} for the two action modes under phase $s^t$, which means "\textbf{the value difference between taking this action and following the current policy}". Here we define the value of speak and skip for two controllers. For the proposer's controller, We set $v_{1,\text{skip}} = +1$ if the pending answer is right otherwise $v_{1,\text{skip}} = -1$. As for speak, $v_{1,\text{speak}} = +1$ if the new answer is right and $v_{1,\text{speak}} = -1$ otherwise. For the reviewer's controller, $v_{2,\text{skip}} = 0$. $v_{2,\text{speak}} = +1$ if the reviewer makes a correct judgement and $v_{2,\text{speak}} = -1$ if the reviewer makes an incorrect judgement. When the judgement is invalid, $v_{2,\text{speak}} = 0$.

Details about how we train $\Pi_1,\Pi_2$ are as follows. Suppose the system is at the $m$-th($m \geq 2$) episode, where 
\[
    m = kT + t,
\] $k$ denoting the samples we have processed and $t$ is the turn in the current sample. We introduce \textbf{instantaneous counterfactual regret} $re^m$ and
\[
    re_i^m(\text{skip},s^m) = v_{i,\text{skip}} - \sum_{a \in \{\text{skip}, \text{speek}\}} \Pi_i^m \left( a \mid s^m \right) v_{i,a}
\]and
\[
    re_i^m(\text{speak},s^m) = v_{i,\text{speak}} - \sum_{a \in \{\text{skip}, \text{speak}\}} \Pi_i^m \left( a \mid s^m \right) v_{i,a},
\]$i = 1$ when $m$ is odd while $i = 2$ when $m$ is even. As for other phases of the current controller and all phases of another controller, there instantaneous counterfactual regret are set to be $0$.

By the design, an action is seen to be "\textbf{good}" when its counterfactual regret is positive, since performing this action obtains a positive utility relative to following the current policy. Otherwise it is seen to be \textbf{bad}. The phase $s^m$ is necessary as it approximately indicates the confidence of the pending answer, and the probability distribution of the controller should be different under different scenarios. Moreover, in the appendix, we will see that it provides the theoretical guarantee. Next we maintain \textbf{cumulative counterfactual regret}:
\[
    Re_i^m(a, \, s) = Re_i^{m - 1}(a, \, s) , + \, re_i^m(a, \, s),
\] for $m \geq 2$ and all phases $s$ and two controllers, where $a \in \{\text{skip},\text{speak}\}$ and $i = 1,2$. This means we accumulate counterfactual regret by phase bucket.
We adopt naive regret matching to update the middle policy only for the current active controller, i.e.:
\[
\Pi_i^{m}\left(a \mid s^m\right) =
\left\{
\begin{array}{ll}
\displaystyle \frac{\max\left(Re_i^{m}(a, s^m), 0\right)}{\sum_{a} \max\left(Re_i^{m}(a, s^m), 0\right)}
& \displaystyle \text{if } \sum_{a} \max\left(Re_i^{m}(a, s^m), 0\right) > 0 \, \text{for } \, a \in \{\text{skip}, \text{speak}\} \\[15pt]
\displaystyle \frac{1}{2}
& \displaystyle \text{otherwise}.
\end{array}
\right.
\]
\begin{wrapfigure}{r}{0.66\textwidth}
\vspace{2pt}
\setlength{\fboxsep}{6pt}
\refstepcounter{algorithm}
\label{alg:three-tier-llm-mas}
\fcolorbox{algoframe}{algobox}{%
\begin{minipage}{0.61\textwidth}
\scriptsize
\textbf{Algorithm~\thealgorithm. TRACER Training for Cooperative Multi-LLM Reasoning}
\vspace{2pt}
\begin{algorithmic}[1]
\Require Agents $U$, horizon $T$, candidates $G$, steps $S$
\Ensure Scheduler $\Pi_1,\Pi_2$, generation policy $\theta$
\State Initialize $\Pi_i$ uniformly over two action modes under all phases.
\State Set cumulative middle and outer regrets to zero.
\For{training step $1,\dots,K$}
    \State Sample tasks and initialize each context $C_1=\{q\}$.
    \For{turn $1,\dots,T$}
        \If{an agent is invoked}
            \State Sample $G$ candidates from $\theta_{i}$ and compute GSPO advantages.
            \State Update $\theta_{i}$ using the clipped sequence-ratio loss.
        \EndIf
        \If{$t \geq 2$}
            \State Compute instantaneous counterfactual.
            \State Update cumulative counterfactual regret and the controller policy $\Pi$.
        \EndIf
        \State Update the phase $s^k$.
        \State Update the context $C_t$ with the help of the two parsers.
        \If{a valid final answer is generated}
            \State \textbf{break}
        \EndIf
    \EndFor
\EndFor
\State \Return averaged $\bar{\Pi_1}^{KT+1},\bar{\Pi_2}^{KT+1}$ and $\theta^{KT+1}$.
\end{algorithmic}
\end{minipage}}
\vspace{-10pt}
\end{wrapfigure}

Note that $\Pi_i$ is only initialized at the very beginning and is updated across the dataset. The cumulative regret is also maintained across the dataset.

\textbf{The phase is determined by visible signals both in the training and reasoning.} \quad $s^t$ is a tiny discrete state key which contains only two signals: pending-score-bucket and proposal-time-bucket. The former denotes the vote of the pending answer and the latter expresses the current round. By the design, we train the controller to act adaptively, i.e., try to keep the optimal answer and challenge the suboptimal answer without an extra neural network, without extra training overhead.

\subsection{End-to-End Training}
In each round, if an agent is invoked, we sample $G$ utterances and update the agent's parameter group by GSPO. Note that these candidates only serve the LLM parameter update and do not enter the genuine training trajectory. For the controllers update, we sample one extra utterance which does not participate in the GSPO update. Only this utterance enters the genuine trajectory after parsing. Finally we adopt
\[
    \bar{\Pi_i}^M \big( a \mid s^t \big) =
    \frac{1}{M_i(s^t)} \sum^{M_i(s^t)}_{m = 1} \Pi_i^m \big( a \mid s^t\big)
\] when $t$ is odd and $t \geq 3$, where 
\[
    M = KT,
\] $K$ being the training samples and $M_i(s^t)$ denotes the number of times the state $s^t$ is visited for the controller $i$ during the total training rounds (i.e., $M$).

Algorithm~\ref{alg:three-tier-llm-mas} summarizes the complete TRACER training pipeline in a compact form. The gray wrapped box emphasizes the executable control flow, while the preceding equations define the exact GSPO objective and the controller-layer regret updates. This layout keeps the method description continuous: the reader can inspect the training loop while the surrounding text explains how the two layers exchange information.

At a high level, each training step samples tasks, lets the controller choose an action mode, and then updates the generation policy whenever the action produces text. The same rollout also provides the counterfactual values used by controller, so the system learns both what to say and when the agents should speak.

\section{Experiments}

We evaluate TRACER on representative reasoning benchmarks and organize the experiments around three questions: (1) whether learned collaboration improves end-task reasoning accuracy; (2) whether the gains remain competitive under realistic inference costs; and (3) whether the collaborative system converges rapidly and stably when being trained. 

\subsection{Experimental Setup}

\paragraph{Overall.}
We train all local RL-style methods on the GSM8K training split only and evaluate on the held-out GSM8K test split, MATH500, and GPQA-D. This design separates in-domain arithmetic reasoning from cross-benchmark out-of-domain generalization: a method must learn collaboration from GSM8K and then transfer without additional task-specific RL on MATH500 or GPQA-D. We organize the study around three questions: whether learned collaboration improves final reasoning accuracy, whether the gain is cost-effective at inference time, and whether multi-turn interaction converges rapidly and stably in training. Detailed dataset counts, decoding budgets, turn horizons, and implementation-specific training settings are reported in Appendix~\ref{sec:experimental_details}.

\paragraph{Baselines.}
We compare against prompting-based and training-based baselines. The non-RL group includes CoT~\citep{wei2022chainofthought}, Self-Consistency~\citep{wang2023selfconsistency}, MAD~\citep{du2023multiagentdebate}, Sparse MAD~\citep{li2024sparsemad}, and MoA~\citep{wang2024mixtureofagents}. The RL group includes Single-Agent GRPO~\citep{shao2024deepseekmath}, Single-Agent GSPO~\citep{zheng2025gspo}, MAPoRL~\citep{park2025maporl}, and MAGRPO~\citep{liu2025magrpo}. We keep these families separated in the tables because TRACER targets learned collaboration, not only prompt-time coordination.

\paragraph{Metrics.}
The primary metric is final task accuracy. We also report inference-time cost using average token usage, model calls, and the number of active agents. Higher accuracy is better, while lower token, call, and active-agent counts indicate lower inference cost.

\subsection{Main Results}

Table~\ref{tab:main_results} summarizes the main benchmark results. The table is grouped by baseline family and backbone so that prompt-only collaboration, single-agent RL, and multi-agent RL are visually separated. This split is intentional: GSM8K measures in-domain arithmetic, whereas MATH500 and GPQA-D test transfer beyond the GSM8K training distribution.

\begin{table}[t]
\caption{\textbf{Main results on representative reasoning benchmarks.} The table is organized by baseline family and reports parallel results under two backbone settings. The Qwen single-agent RL entries use completed local GSM8K-trained final checkpoints evaluated under the same protocol.}
\vspace{2pt}
\label{tab:main_results}
\centering
\footnotesize
\renewcommand{\arraystretch}{1.08}
\setlength{\tabcolsep}{3pt}
\resizebox{\columnwidth}{!}{%
\begin{tabular}{l:c:c:c:c:c:c:c:c}
\specialrule{1.1pt}{0pt}{0pt}
\rowcolor{tableheader}
\multicolumn{1}{c}{\textbf{Method}} & \multicolumn{4}{c}{\textbf{Phi-3 Mini 4K Instruct}} & \multicolumn{4}{c}{\textbf{Qwen2.5-7B-Instruct}} \\
\cmidrule(lr){2-5}\cmidrule(lr){6-9}
\rowcolor{tablesubheader}
 & \textbf{GSM8K $\uparrow$} & \textbf{MATH500 $\uparrow$} & \textbf{GPQA-D $\uparrow$} & \textbf{Avg. $\uparrow$} & \textbf{GSM8K $\uparrow$} & \textbf{MATH500 $\uparrow$} & \textbf{GPQA-D $\uparrow$} & \textbf{Avg. $\uparrow$} \\
\hdashline
\specialrule{0.8pt}{0pt}{0pt}
\tablesectionrow{9}{(a) Non-RL Baselines}
\methodref{CoT}{wei2022chainofthought} & 0.8250 & \rankone{0.4130} & \rankone{0.3280} & \ranktwo{0.5220} & \ranktwo{0.9160} & \rankone{0.7550} & \rankone{0.3640} & \rankone{0.6783} \\
\hdashline
\methodref{Self-Consistency}{wang2023selfconsistency} & \rankone{0.9037} & 0.3580 & 0.3080 & \rankone{0.5232} & \rankone{0.9240} & \ranktwo{0.6760} & 0.2727 & \ranktwo{0.6242} \\
\hdashline
\methodref{MAD}{du2023multiagentdebate} & 0.8408 & 0.3740 & \ranktwo{0.3182} & 0.5110 & 0.8862 & 0.4760 & 0.3210 & 0.5611 \\
\hdashline
\methodref{Sparse MAD}{li2024sparsemad} & 0.8469 & 0.3740 & 0.3131 & 0.5113 & 0.8878 & 0.4780 & 0.3081 & 0.5580 \\
\hdashline
\methodref{MoA}{wang2024mixtureofagents} & 0.8044 & 0.3660 & 0.2475 & 0.4726 & 0.8870 & 0.4320 & 0.2828 & 0.5339 \\
\hdashline
\rowcolor{tablehighlight}
\ourscell{TRACER (ours)} &
\oursval{\ranktwo{0.8802}} & \oursval{\ranktwo{0.4000}} & \oursval{0.2727} & \oursval{0.5027} & \oursval{0.8901} & \oursval{0.6120} & \oursval{\ranktwo{0.3535}} & \oursval{0.6185} \\
\hdashline
\specialrule{0.6pt}{0pt}{0pt}
\tablesectionrow{9}{(b) RL Baselines}
\methodref{Single-Agent GRPO}{shao2024deepseekmath} & 0.8681 & \rankone{0.4360} & 0.0947 & 0.4663 & 0.8825 & \ranktwo{0.4780} & 0.0341 & \ranktwo{0.4649} \\
\hdashline
\methodref{Single-Agent GSPO}{zheng2025gspo} & \ranktwo{0.8795} & \ranktwo{0.4100} & 0.0821 & 0.4572 & 0.8741 & 0.4580 & 0.0366 & 0.4563 \\
\hdashline
\methodref{MAPoRL}{park2025maporl} & 0.8415 & 0.3600 & \ranktwo{0.2664} & \ranktwo{0.4893} & 0.8385 & 0.1000 & \ranktwo{0.2071} & 0.3819 \\
\hdashline
\methodref{MAGRPO}{liu2025magrpo} & 0.7252 & 0.3420 & 0.0884 & 0.3852 & 0.6960 & 0.3270 & 0.0189 & 0.3473 \\
\hdashline
\specialrule{0.6pt}{0pt}{0pt}
\rowcolor{tablehighlight}
\ourscell{TRACER (ours)} & \oursval{\rankone{0.8802}} & \oursval{0.4000} & \oursval{\rankone{0.2727}} & \oursval{\rankone{0.5027}} & \oursval{\rankone{{0.8901}}} & \oursval{\rankone{0.6120}} & \oursval{\rankone{0.3535}} & \oursval{\rankone{0.6185}} \\
\hdashline
\specialrule{1.1pt}{0pt}{0pt}
\end{tabular}
}
\end{table}

\begin{table}[t]
\caption{\textbf{Inference-time cost comparison under matched evaluation settings.} Accuracy repeats the full GSM8K test accuracy from Table~\ref{tab:main_results} where available, while token, call, and active-agent costs are measured on the full GSM8K test split. The Qwen single-agent RL costs are measured from the completed local full-test traces.}
\vspace{2pt}
\label{tab:cost_fairness}
\centering
\scriptsize
\renewcommand{\arraystretch}{0.96}
\setlength{\tabcolsep}{2.2pt}
\resizebox{\columnwidth}{!}{%
\begin{tabular}{:l:c:c:c:c:c:c:c:c:}
\specialrule{1.1pt}{0pt}{0pt}
\rowcolor{tableheader}
\multirow{2}{*}{\centering\textbf{Method}} & \multicolumn{4}{c}{\textbf{Phi-3 Mini 4K Instruct}} & \multicolumn{4}{c}{\textbf{Qwen2.5-7B-Instruct}} \\
\cmidrule(lr){2-5}\cmidrule(lr){6-9}
\rowcolor{tablesubheader}
 & \textbf{GSM8K $\uparrow$} & \textbf{Tok. $\downarrow$} & \textbf{Calls $\downarrow$} & \textbf{Agents $\downarrow$} & \textbf{GSM8K $\uparrow$} & \textbf{Tok. $\downarrow$} & \textbf{Calls $\downarrow$} & \textbf{Agents $\downarrow$} \\
\hdashline
\specialrule{0.8pt}{0pt}{0pt}
\tablesectionrow{9}{(a) Non-RL Baselines}
\methodref{CoT}{wei2022chainofthought} & 0.8250 & \ranktwo{1461.1} & \rankone{1.00} & \rankone{1} & \ranktwo{0.9160} & \ranktwo{1355.1} & \rankone{1.00} & \rankone{1} \\
\hdashline
\methodref{Self-Consistency}{wang2023selfconsistency} & \rankone{0.9037} & 14610.8 & 10.00 & \rankone{1} & \rankone{0.9240} & 13551.4 & 10.00 & \rankone{1} \\
\hdashline
\methodref{MAD}{du2023multiagentdebate} & 0.8408 & 5345.0 & 6.00 & 3 & 0.8862 & 5909.1 & 6.00 & 3 \\
\hdashline
\methodref{Sparse MAD}{li2024sparsemad} & 0.8469 & 4957.0 & 6.00 & 3 & 0.8878 & 5488.1 & 6.00 & 3 \\
\hdashline
\methodref{MoA}{wang2024mixtureofagents} & 0.8044 & 7995.4 & 12.00 & \rankone{1} & 0.8870 & 8408.5 & 12.00 & \rankone{1} \\
\hdashline
\ourscell{TRACER (ours)} &
\oursval{\ranktwo{0.8802}} & \oursval{\rankone{959.7}} & \oursval{\ranktwo{2.94}} & \oursval{\ranktwo{2}} & \oursval{0.8901} & \oursval{\rankone{1014.1}} & \oursval{\ranktwo{3.02}} & \oursval{\ranktwo{2}} \\
\hdashline
\specialrule{0.6pt}{0pt}{0pt}
\tablesectionrow{9}{(b) RL Baselines}
\methodref{Single-Agent GRPO}{shao2024deepseekmath} & 0.8681 & \ranktwo{282.1} & \rankone{1.00} & \rankone{1} & \ranktwo{0.8825} & \ranktwo{252.3} & \rankone{1.00} & \rankone{1} \\
\hdashline
\methodref{Single-Agent GSPO}{zheng2025gspo} & \ranktwo{0.8795} & \rankone{276.2} & \rankone{1.00} & \rankone{1} & 0.8741 & \rankone{246.2} & \rankone{1.00} & \rankone{1} \\
\hdashline
\methodref{MAPoRL}{park2025maporl} & 0.8415 & 1770.7 & 4.00 & \ranktwo{2} & 0.8385 & 2098.9 & 4.00 & \ranktwo{2} \\
\hdashline
\methodref{MAGRPO}{liu2025magrpo} & 0.7252 & 1267.1 & 4.00 & \ranktwo{2} & 0.6960 & 1297.0 & 4.00 & \ranktwo{2} \\
\hdashline
\specialrule{0.6pt}{0pt}{0pt}
\rowcolor{tablehighlight}
\ourscell{TRACER (ours)} & \oursval{\rankone{0.8802}} & \oursval{959.7} & \oursval{\ranktwo{2.94}} & \oursval{\ranktwo{2}} & \oursval{\rankone{0.8901}} & \oursval{1014.1} & \oursval{\ranktwo{3.02}} & \oursval{\ranktwo{2}} \\
\hdashline
\specialrule{1.1pt}{0pt}{0pt}
\end{tabular}
}
\end{table}

\vspace{-5pt}
\subsection{Cost Fairness}
\vspace{-3pt}

Table~\ref{tab:cost_fairness} compares the inference cost of different collaboration strategies. The accuracy columns repeat the full GSM8K test accuracy from Table~\ref{tab:main_results}; the cost columns use the full GSM8K test split whenever the corresponding full prediction trace records token accounting. Here \textbf{Tok.} denotes the average total prompt-plus-completion tokens consumed per full-test example across all agents and all turns, \textbf{Calls} counts the total number of model invocations per example, and \textbf{Agents} counts how many distinct agents actually speak.

\begin{table}[t!]
\caption{\textbf{TRACER ablation suite on GSM8K.} Rows remove or freeze one component; columns report GSM8K full-test accuracy and matched inference-time cost for each backbone. Values highlighted in blue represent the best performance. }
\vspace{0pt}
\label{tab:tracer_ablation_suite}
\centering
\scriptsize
\renewcommand{\arraystretch}{1.04}
\setlength{\tabcolsep}{2.2pt}
\resizebox{\columnwidth}{!}{%
\begin{tabular}{lcccccccc}
\hline
\rowcolor{tableheader}
\multicolumn{1}{l}{\textbf{Variant}} & \multicolumn{4}{c}{\textbf{Phi-3 Mini 4K Instruct}} & \multicolumn{4}{c}{\textbf{Qwen2.5-7B-Instruct}} \\
\hline
\rowcolor{tablesubheader}
 & \textbf{Acc. $\uparrow$} & \textbf{Tok. $\downarrow$} & \textbf{Calls. $\downarrow$} & \textbf{Agents. $\downarrow$} & \textbf{Acc. $\uparrow$} & \textbf{Tok. $\downarrow$} & \textbf{Calls $\downarrow$} & \textbf{Agents. $\downarrow$} \\
\hline
Full TRACER & \rankone{0.8802} & \rankone{959.7} & \rankone{2.94} & 2 & \rankone{0.8901} & \rankone{1014.1} & \rankone{3.02} & 2 \\
\hline
w/o all learned updates & \ablcell{0.8453}{\ablbad{-0.0349}} & \ablcell{1097.1}{\ablbad{+137.4}} & \ablcell{3.00}{\ablbad{+0.06}} & \ablcell{2}{\ablzero{+0}} & \ablcell{0.8848}{\ablbad{-0.0053}} & \ablcell{1023.2}{\ablbad{+9.1}} & \ablcell{3.00}{\ablgood{-0.02}} & \ablcell{2}{\ablzero{+0}} \\
\hline
w/o reviewer & \ablcell{0.7741}{\ablbad{-0.1061}} & \ablcell{2156.73}{\ablbad{+1197.03}} & \ablcell{5.00}{\ablbad{+2.06}} & \ablcell{1}{\ablgood{-1}} & \ablcell{0.8241}{\ablbad{-0.0660}} & \ablcell{2215.47}{\ablbad{+1201.37}} & \ablcell{5}{\ablbad{+1.98}} & \ablcell{1}{\ablgood{-1}} \\
\hline
w/o reviewer controller & \ablcell{0.8052}{\ablbad{-0.0750}} & \ablcell{1412.35}{\ablbad{+452.65}} & \ablcell{3.72}{\ablbad{+0.78}} & \ablcell{2}{\ablzero{+0}} & \ablcell{0.8582}{\ablbad{-0.0319}} & \ablcell{1458.62}{\ablbad{+444.52}} & \ablcell{3.81}{\ablbad{+0.79}} & \ablcell{2}{\ablzero{+0}} \\
\hline
w/o proposer controller & \ablcell{0.7650}{\ablbad{-0.1152}} & \ablcell{1468.92}{\ablbad{+509.22}} & \ablcell{3.85}{\ablbad{+0.91}} & \ablcell{2}{\ablzero{+0}} & \ablcell{0.8150}{\ablbad{-0.0751}} & \ablcell{1512.45}{\ablbad{+498.35}} & \ablcell{3.93}{\ablbad{+0.91}} & \ablcell{2}{\ablzero{+0}} \\
\hline
random reviewer controller & \ablcell{0.8211}{\ablbad{-0.0591}} & \ablcell{1567.28}{\ablbad{+607.58}} & \ablcell{3.92}{\ablbad{+0.98}} & \ablcell{2}{\ablzero{+0}} & \ablcell{0.8749}{\ablbad{-0.0152}} & \ablcell{1623.18}{\ablbad{+609.08}} & \ablcell{3.99}{\ablbad{+0.97}} & \ablcell{2}{\ablzero{+0}} \\
\hline
random proposer controller & \ablcell{0.7832}{\ablbad{-0.0970}} & \ablcell{1554.13}{\ablbad{+594.43}} & \ablcell{3.97}{\ablbad{+1.03}} & \ablcell{2}{\ablzero{+0}} & \ablcell{0.8347}{\ablbad{-0.0554}} & \ablcell{1605.73}{\ablbad{+591.63}} & \ablcell{4.05}{\ablbad{+1.03}} & \ablcell{2}{\ablzero{+0}} \\
\hline
w/o vote state & \ablcell{0.8097}{\ablbad{-0.0705}} & \ablcell{1258.46}{\ablbad{+298.76}} & \ablcell{3.21}{\ablbad{+0.27}} & \ablcell{2}{\ablzero{+0}} & \ablcell{0.8628}{\ablbad{-0.0273}} & \ablcell{1302.54}{\ablbad{+288.44}} & \ablcell{3.32}{\ablbad{+0.30}} & \ablcell{2}{\ablzero{+0}} \\
\hline
w/o iteration state & \ablcell{0.7892}{\ablbad{-0.0910}} & \ablcell{1330.96}{\ablbad{+371.26}} & \ablcell{3.46}{\ablbad{+0.52}} & \ablcell{2}{\ablzero{+0}} & \ablcell{0.8408}{\ablbad{-0.0493}} & \ablcell{1385.27}{\ablbad{+371.17}} & \ablcell{3.55}{\ablbad{+0.53}} & \ablcell{2}{\ablzero{+0}} \\
\hline
w/o reviewer GSPO update & \ablcell{0.7695}{\ablbad{-0.1107}} & \ablcell{1335.20}{\ablbad{+375.50}} & \ablcell{3.55}{\ablbad{+0.61}} & \ablcell{2}{\ablzero{+0}} & \ablcell{0.8195}{\ablbad{-0.0706}} & \ablcell{1395.30}{\ablbad{+381.20}} & \ablcell{3.64}{\ablbad{+0.62}} & \ablcell{2}{\ablzero{+0}} \\
\hline
    w/o proposer GSPO update & \ablcell{0.7703}{\ablbad{-0.1099}} & \ablcell{1333.04}{\ablbad{+373.34}} & \ablcell{3.52}{\ablbad{+0.58}} & \ablcell{2}{\ablzero{+0}} & \ablcell{0.8203}{\ablbad{-0.0698}} & \ablcell{1390.20}{\ablbad{+376.10}} & \ablcell{3.58}{\ablbad{+0.56}} & \ablcell{2}{\ablzero{+0}} \\
\hline
\end{tabular}
}
\end{table}

\vspace{-5pt}

\subsection{Ablation Study}
\vspace{-3pt}

Table~\ref{tab:tracer_ablation_suite} mirrors the final ablation entry points implemented in the TRACER training suite. We use compact \emph{w/o} notation for removed components so that the table remains comparable to the main result tables. The ablation suite is evaluated on the full GSM8K test split by default, since its goal is to isolate which learned collaboration components affect in-domain training behavior rather than to repeat the cross-benchmark transfer study. 
\
\vspace{-5pt}
\subsection{Training Dynamics}
\vspace{-3pt}

Figure~\ref{fig:training_dynamics_main} summarizes the GSM8K training traces for the completed matched baseline runs under both backbones and appends the currently available TRACER/Ours trace on the right. The first four columns retain the original baseline diagnostics: final-turn correctness for MAPoRL, checkpoint-level sampled evaluation for MAGRPO when the online reward log is uninformative, and mean reward for single-agent RL baselines. The rightmost panels show TRACER/Ours running R5 accuracy over the first 1,000 sampled training examples for both backbones. These curves are run-level diagnostics rather than held-out benchmark claims; the authoritative GSM8K, MATH500, and GPQA-D results remain those in Table~\ref{tab:main_results}.

\vspace{-5pt}

\begin{figure}[H]
\centering
\includegraphics[width=\columnwidth,height=0.36\columnwidth]{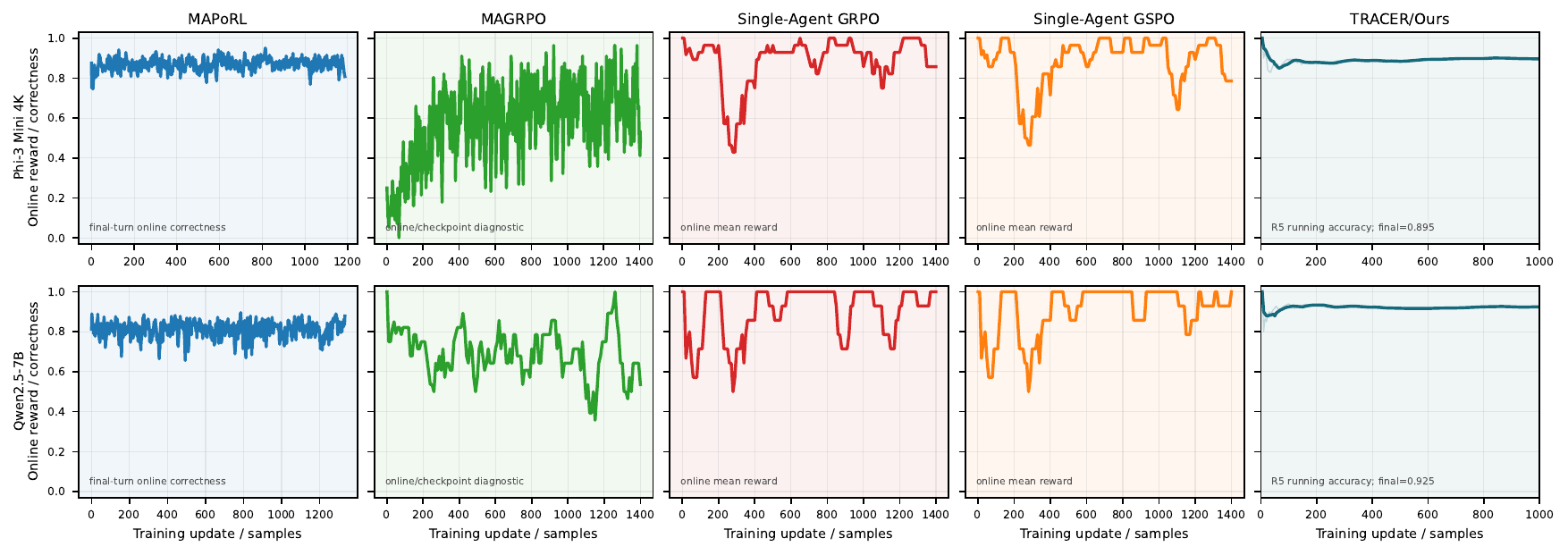}
\caption{\textbf{GSM8K training dynamics for completed matched methods.}}
\label{fig:training_dynamics_main}
\end{figure}

\vspace{-8pt}
\section{Conclusion and Limitations}
\vspace{-3pt}

  This paper presents TRACER, a two-layer framework for cooperative multi-LLM reasoning that jointly learns action-mode selection and role-specific generation optimization. TRACER combines GSPO-based inner-loop optimization with counterfactual credit assignment and regret-matching updates, moving beyond fixed debate or voting protocols toward adaptive collaboration policies. The empirical study evaluates both accuracy and inference cost under a unified protocol. By carefully designing a binary action space, namely speaking versus skipping, we formulate multi-agent collaboration as a game-theoretic model with a finite action space. This formulation provides mathematically rigorous convergence guarantees, fast and stable training dynamics. Compared with both non-RL and RL baselines, TRACER achieves a favorable trade-off performance between accuracy and inference cost across multiple benchmarks. Moreover, the framework can be readily extended to scenarios involving more agents only by adding corresponding initialized controllers. Limitations remain: training is currently GSM8K-only, transfer is tested on MATH500 and GPQA-D. Future work will broaden domains, horizons, and backbones while preserving explicit cost accounting.

\clearpage
\bibliographystyle{plainnat}
\bibliography{references}

\appendix
\renewcommand{\thetable}{A\arabic{table}}
\renewcommand{\theHtable}{A\arabic{table}}
\setcounter{table}{0}
\renewcommand{\thealgorithm}{A\arabic{algorithm}}
\providecommand{\theHalgorithm}{}
\renewcommand{\theHalgorithm}{A\arabic{algorithm}}
\setcounter{algorithm}{0}
\section*{Appendix Contents}
\label{sec:appendix_contents}

\noindent\begin{tabular}{p{0.22\columnwidth}p{0.70\columnwidth}}
\textbf{\hyperref[sec:experimental_details]{Appendix A}} & Experimental details: datasets, split policy, backbone settings, turn budgets, decoding, parsing, and evaluation counts. \\
\textbf{\hyperref[sec:prompt_templates]{Appendix B}} & Prompt templates: local prompt templates for non-RL baselines, RL baselines, and TRACER. \\
\textbf{\hyperref[sec:full_pseudocode]{Appendix C}} & Full TRACER pseudocode: expanded algorithmic flow for the controller and generation-credit updates. \\
\textbf{\hyperref[sec:convergence_proof]{Appendix D}} & Convergence proof: CFR-style regret definitions and convergence argument. \\
\end{tabular}

\vspace{6pt}

\section{Experimental Details}
\label{sec:experimental_details}

This appendix section reports the training hyperparameters, inference configuration, answer parsing, verification settings, and benchmark-specific preprocessing used for the matched local runs. Several baselines use method-specific implementations, so we report the actual local run settings rather than forcing all methods into a single nominal configuration.

\paragraph{Datasets and splits.}
All local MAPoRL, MAGRPO, single-agent RL, and TRACER runs use the GSM8K training split for optimization. Held-out evaluation uses the GSM8K test split for in-domain accuracy and MATH500 plus GPQA-Diamond for cross-benchmark generalization. GPQA-Diamond is evaluated with four shuffled option-order repeats per question.

\paragraph{Backbones and interaction budgets.}
We use Phi-3 Mini 4K Instruct and Qwen2.5-7B-Instruct as the two backbone families. Agents within the same run share the same backbone family. Multi-Agent RL methods share the same adaptive five runs both in training and reasoning.

\begin{table}[H]
\caption{\textbf{Training and evaluation configuration.} Local RL runs use the GSM8K training split for optimization and evaluate on GSM8K test, MATH500, and GPQA-D without task-specific fine-tuning on the latter two benchmarks.}
\label{tab:appendix_exp_details}
\centering
\footnotesize
\renewcommand{\arraystretch}{1.08}
\setlength{\tabcolsep}{3pt}
\begin{tabular}{p{0.14\columnwidth}p{0.19\columnwidth}p{0.58\columnwidth}}
\specialrule{1.1pt}{0pt}{0pt}
\rowcolor{tableheader}
\textbf{Category} & \textbf{Item} & \textbf{Value} \\
\specialrule{0.8pt}{0pt}{0pt}
\tablesectionrow{3}{Training}
Training set & dataset / split & GSM8K / train \\
Optimizer & optimizer / learning rate & MAPoRL, MAGRPO, GRPO-Phi, GSPO-Phi: Adafactor / $1\times10^{-5}$; Qwen single-agent fp32: AdamW; TRACER: \texttt{torch.optim.Adam} \\
Batching & batch size / accumulation & Per-device batch size 1; MAPoRL/MAGRPO accumulation 4; GRPO-Phi/GSPO-Phi accumulation 8; TRACER uses its native online update loop \\
Turn horizon & learned collaboration budget & MAPoRL/MAGRPO: 2 turns; TRACER: adaptive 5-turn horizon \\
Rollout & candidate count & MAPoRL PPO: \texttt{num\_sample\_generations}=20; MAGRPO-Phi: 4; MAGRPO-Qwen: 2; single-agent Phi GRPO/GSPO: 4; TRACER candidate chunks: 4 \\
Initialization & backbone / number of agents & Phi-3 Mini 4K or Qwen2.5-7B / 2 agents \\
\specialrule{0.6pt}{0pt}{0pt}
\tablesectionrow{3}{Inference}
Training decoding & temperature / top-$p$ / max tokens & MAPoRL/MAGRPO/GRPO/GSPO: 0.7 / 0.95; MAGRPO training max tokens 64; single-agent Phi training max tokens 300; TRACER temperature 0.3 \\
Evaluation decoding & max tokens & Default local math eval uses 300 when not overridden; MAGRPO-Phi GPQA-D uses 160; MAGRPO-Qwen MATH500 uses the 384-token rerun; single-agent Phi/Qwen MATH500 uses 512; Qwen single-agent GPQA-D uses 300 \\
Stopping & stop rule / max rounds & stop after a valid final answer or at the configured turn limit \\
Aggregation & final answer selection rule & normalized final answer; majority vote for multi-sample statistics \\
\specialrule{0.6pt}{0pt}{0pt}
\tablesectionrow{3}{Evaluation}
Parser & answer extraction rule & numeric extraction for GSM8K/MATH500; option extraction for GPQA-D \\
Verifier & exact-match / option-match rule & normalized exact match for math; option match for GPQA-D \\
Benchmark & split / sample count / repeat setting & GSM8K full test (1319); MATH500 full test (500); GPQA-Diamond uses 198 questions $\times$ 4 shuffled option-order repeats = 792 local evaluation records \\
Main-table policy & sampled diagnostics & Sampled-200 diagnostics and truncation-failure diagnostics are not used for main tables \\
\specialrule{1.1pt}{0pt}{0pt}
\end{tabular}
\end{table}

The MAGRPO+Qwen2.5-7B MATH500 value in Table~\ref{tab:main_results} is the full MATH500 rerun of the same GSM8K-trained checkpoint with \texttt{max\_new\_tokens}=384, giving \texttt{final\_turn\_accuracy}=0.327. The earlier default 64-token MATH500 diagnostic produced 0.0000 because the generations were truncated, so it is not used as the reported benchmark result.

\paragraph{Experiments compute resources.}8*H20s with LoRa adaptation both in training and reasoning.

\section{Prompt Templates}
\label{sec:prompt_templates}

This section records the prompt templates used for the local baseline and TRACER evaluations. We include prompts for non-RL baselines, MAPoRL/MAGRPO, single-agent RL baselines, and the TRACER proposal-review pipeline.

\begin{table}[H]
\caption{\textbf{Prompt source map.} The table maps each method family to the prompt source used in the local evaluation protocol.}
\label{tab:prompt_source_map}
\centering
\scriptsize
\renewcommand{\arraystretch}{1.05}
\setlength{\tabcolsep}{3pt}
\begin{tabular}{p{0.19\columnwidth}p{0.20\columnwidth}p{0.48\columnwidth}}
\specialrule{1.1pt}{0pt}{0pt}
\rowcolor{tableheader}
\textbf{Method family} & \textbf{Status} & \textbf{Local source} \\
\specialrule{0.8pt}{0pt}{0pt}
CoT / Self-Consistency / MAD / Sparse MAD / MoA & Listed & Appendix prompt templates below \\
\hdashline
Single-Agent GRPO / GSPO & Recovered & \path{MAPoRL/trl/trl/trainer/utils_multi_unified_chat.py}; same GSM8K formatting function as local RL training. \\
\hdashline
MAPoRL / MAGRPO & Recovered & \path{MAPoRL/trl/trl/trainer/utils_multi_unified_chat.py}, \path{MAPoRL/eval_cross_maporl.py}. \\
\hdashline
TRACER & Recovered & \path{Three-Layer.../.worktrees/final/src/config.py}. \\
\specialrule{1.1pt}{0pt}{0pt}
\end{tabular}
\end{table}

\subsection{Non-RL Baseline Prompt Templates}

\begin{promptbox}{CoT prompt}
Solve the following grade school math problems.Show your reasoning step by step.End each solution with 'Therefore, the answer is <number>.'

Q: Natalia sold clips to 48 of her friends in April, and then she sold half as many clips in May. How many clips did Natalia sell altogether in April and May?

A: Natalia sold 48/2 = 24 clips in May. Natalia sold 48+24 = 72 clips altogether in April and May.Therefore, the answer is 72.

Q: Weng earns \$12 an hour for babysitting. Yesterday, she just did 50 minutes of babysitting. How much did she earn?

A: Weng earns 12/60 = \$0.2 per minute. Working 50 minutes, she earned 0.2 x 50 = \$10. Therefore, the answer is 10.

Q: Betty is saving money for a new wallet which costs \$100. Betty has only half of the money she needs. Her parents decided to give her \$15 for that purpose, and her grandparents twice as much as her parents. How much more money does Betty need to buy the wallet?

A: In the beginning, Betty has only 100 / 2 = \$50. Betty's grandparents gave her 15 * 2 = \$30. This means, Betty needs 100 - 50 - 30 - 15 = \$5 more. Therefore, the answer is 5.

Q: Julie is reading a 120-page book. Yesterday, she was able to read 12 pages and today, she read twice as many pages as yesterday. If she wants to read half of the remaining pages tomorrow, how many pages should she read?

A: Maila read 12 x 2 = 24 pages today. So she was able to read a total of 12 + 24 = 36 pages since yesterday. There are 120 - 36 = 84 pages left to be read. Since she wants to read half of the remaining pages tomorrow, then she should read 84/2 = 42 pages. Therefore, the answer is 42.

Q: James writes a 3-page letter to 2 different friends twice a week. How many pages does he write a year?

A: He writes each friend 32=6 pages a week So he writes 62=12 pages every week That means he writes 12*52=624 pages a year Therefore, the answer is 624.

Q: Mark has a garden with flowers. He planted plants of three different colors in it. Ten of them are yellow, and there are 80\% more of those in purple. There are only 25\% as many green flowers as there are yellow and purple flowers. How many flowers does Mark have in his garden?

A: There are 80/100 * 10 = 8 more purple flowers than yellow flowers. So in Mark's garden, there are 10 + 8 = 18 purple flowers. Purple and yellow flowers sum up to 10 + 18 = 28 flowers. That means in Mark's garden there are 25/100 * 28 = 7 green flowers. So in total Mark has 28 + 7 = 35 plants in his garden. Therefore, the answer is 35.

Q: Albert is wondering how much pizza he can eat in one day. He buys 2 large pizzas and 2 small pizzas. A large pizza has 16 slices and a small pizza has 8 slices. If he eats it all, how many pieces does he eat that day?

A: He eats 32 from the largest pizzas because 2 x 16 = 32 He eats 16 from the small pizza because 2 x 8 = 16 He eats 48 pieces because 32 + 16 = 48 Therefore, the answer is 48.

Q: Randy has 60 mango trees on his farm. He also has 5 less than half as many coconut trees as mango trees. How many trees does Randy have in all on his farm?

A: Half of the number of Randy's mango trees is 60/2 = 30 trees. So Randy has 30 - 5 = 25 coconut trees. Therefore, Randy has 60 + 25 = 85 treeson his farm. Therefore, the answer is 85.

Q: {question} 

A: Let's think step by step.
\end{promptbox}

\begin{promptbox}{Self-Consistency prompt}
Adopt the same 8-shot prompt as CoT. There is no extra prompt guiding the model to vote and we set the number of samples to be 10.
\end{promptbox}

\begin{promptbox}{MAD / Sparse MAD prompt}
Initial prompt:

GSM8K: Can you solve the following math problem? {question}
Explain your reasoning.
Your final answer should be a single numerical number, in the form \boxed{answer}, at the end of your response.

MATH500/GPQA-D: From dataset keys.

Debate prompt:

GSM8K: These are the solutions to the problem from other agents:

One agent solution: ```{solution_1}'''

One agent solution: ```{solution_2}'''

Using the solutions from other agents as additional information, can you provide your answer to the math problem?
The original math problem is {original_problem}
Your final answer should be a single numerical number, in the form \boxed{answer}, at the end of your response.

MATH500: These are the solutions to the problem from other agents:
One agent solution: {solution_i}
Using the solutions from other agents as additional information, can you provide your answer to the math problem?The original math problem is {original_problem}Your final answer should be placed within \boxed{} at the end of your response.

GPQA-D: These are the solutions to the problem from other agents:
One agent solution: {solution_i}
Using the reasoning from other agents as additional advice, can you give an updated answer? Examine your solution and that other agents step by step.The original problem is {original_problem}Put your final answer in the form \boxed{A}, \boxed{B}, \boxed{C}, or \boxed{D} at the end of your response.

If no other solution is visible, the fallback self-check prompts are:

GSM8K: Can you double check that your answer is correct.Please reiterate your answer, with your final answer a single numerical number, in the form \boxed{answer}.

MATH500: Can you double check that your answer is correct.Please reiterate your answer, with your final answer placed within \boxed{} at the end of your response.

GPQA-D: Can you double check that your answer is correct.Put your final answer in the form \boxed{A}, \boxed{B}, \boxed{C}, or \boxed{D} at the end of your response.
\end{promptbox}

\begin{promptbox}{MoA prompt}
Layer0 Proposer Prompt:

GSM8K: Can you solve the following math problem? {question}Explain your reasoning.Your final answer should be a single numerical number, in the form \boxed {answer}, at the end of your response.

MATH500/GPQA-D: From dataset keys

Aggregator Prompt:

You are participating in a Mixture-of-Agents reasoning pipeline.You are the {stage_label} at layer {layer_index + 1} of {num_layers}.You will receive multiple candidate solutions to the same problem.Compare them carefully, synthesize the strongest reasoning, discard unsupported claims, and produce one improved answer.
Original problem:{original_problem}
Candidate solutions:

Candidate 1: {candidate_1}

Candidate 2: {candidate_2}

Candidate 3: {candidate_3}

{final_instruction}. In the middle layer, stage_label = "intermediate aggregater". In the final layer, stage_label = "final aggregator".

Final instruction:

GSM8K: Provide a single consolidated solution. Your final answer should be a single numerical number, in the form \boxed{answer}, at the end of your response.

MATH500: Provide a single consolidated solution. Your final answer should be placed within \boxed{} at the end of your response.

GPQA-D: Provide a single consolidated answer. Put your final answer in the form \boxed{A}, \boxed{B}, \boxed{C}, or \boxed{D} at the end of your response.
\end{promptbox}

\subsection{RL Baseline and Evaluation Prompts}

\begin{promptbox}{GSM8K Phi-3 prompt for MAPoRL / MAGRPO / single-agent RL}
Question: {question}

Provide a reasoning for your solution. At the end, you MUST write the answer in the following format:```

Answer: \boxed{ XXX }'''

Please ensure that the final answer is always formatted this way.
\end{promptbox}

\begin{promptbox}{GSM8K Qwen2.5 prompt for MAPoRL / MAGRPO / single-agent RL}
Question: {question}

Provide a very short and precise solution in at most two short sentences. Do not restate the question. At the end, you MUST write the final answer in the following format:

Answer: \boxed{XX}

Please ensure that the final answer is always formatted this way.
\end{promptbox}

\begin{promptbox}{GSM8K multi-agent follow-up prompt for MAPoRL / MAGRPO}
These are the final answers proposed by other agents:

Agent {i} proposed final answer: {candidate_answer}

Use the other agents' final answers only as hints. Re-solve the problem yourself with very short reasoning. Do not copy or quote the other agents' solutions. At the end, write exactly one final answer in the format: Answer: \boxed{XX}.

Once again, the question is: {question_for_input}
\end{promptbox}

\begin{promptbox}{MATH500 evaluation prompt}
Question: {question}

Solve the problem step by step. At the end, you MUST write the final answer in exactly the following format:

Answer: \boxed{YOUR_ANSWER}

Do not omit the boxed final answer.
\end{promptbox}

\begin{promptbox}{GPQA-Diamond evaluation prompt}
Answer the following multiple-choice question.

Question: {question}

Choices:
A. {choice_A}
B. {choice_B}
C. {choice_C}
D. {choice_D}

Think step by step. At the end, you MUST write the single best answer in exactly the following format:

Answer: \boxed{A}

Replace A with one of A, B, C, or D.
\end{promptbox}

\subsection{TRACER Proposal-Review Prompts}

\begin{promptbox}{TRACER reviewer prompt}
{context}

You are a careful reasoning assistant.
Check whether the current pending final answer is correct for the original math problem.
First line: RIGHT or WRONG.
Second line: one short reason.
Do not write a new solution or a new final answer.

Review:
\end{promptbox}

\begin{promptbox}{TRACER proposer prompt}
{context}

Solve the original math problem directly.
If there is a pending solution or short review feedback, use it only as a hint.
Keep the reasoning concise.
End with exactly one final line: Final answer: <number>.

Solution:
\end{promptbox}

\begin{promptbox}{TRACER MATH500 proposer prompt}
{context}

Solve the original math problem directly.
If there is a pending solution or short review feedback, use it only as a hint.
Keep one concise chain of thought; do not branch into alternatives.
Use the exact answer form required by the problem.
End with exactly one final line: Final answer: \boxed{<answer>}.

Solution:
\end{promptbox}

\begin{promptbox}{TRACER GPQA-Diamond proposer prompt}
{context}

Solve the original multiple-choice science question directly.
If there is a pending solution or short review feedback, use it only as a hint.
Keep one concise chain of thought; do not branch into alternatives.
Choose exactly one option letter from A, B, C, and D.
End with exactly one final line: Final answer: <A/B/C/D>.

Solution:
\end{promptbox}

\section{Full TRACER Training Pseudocode}
\label{sec:full_pseudocode}

The main text uses a compact wrapped version of the training loop to save space. Algorithm~\ref{alg:full-three-tier-llm-mas} gives the expanded pseudocode, including the role-specific GSPO update and controller-layer regret matching.

\refstepcounter{algorithm}
\label{alg:full-three-tier-llm-mas}
\noindent\textbf{Algorithm~\thealgorithm. Full TRACER Training for Cooperative Multi-LLM Reasoning}

{\scriptsize
\begin{algorithmic}[1]
\Require Two pretrained agents, turn horizon $T$, candidate count $G$, training steps $K$, batch size $E$, evaluator $R(\cdot,y^*)\in[0,1]$, constants $\epsilon,\epsilon'$.
\Ensure Trained Controller $\Pi_1,\Pi_2$ and inner generation policy $\theta_1,\theta_2$
\State Initialize $\Pi_i^0(a\mid \text{phase})=\frac{1}{2}$ for $a\in\{\text{skip},\text{speak}\},i = 1,2$ and all phases.
\State Initialize cumulative regrets $Re_i^0(\cdot,\cdot)\gets 0$, $i = 1,2$.
\State Initialize $M_i(\cdot) = 0$ for $i = 1,2$ and all phases.
\For{$k=1,\dots,K$}
    \State Sample $E$ training problems and initialize each environment context as $C_e=\{q\}$.
    \State Initialize phase-specific accumulators $\Pi_i^1(\cdot \mid \cdot)\gets \frac{1}{2}, Re_i^1(\cdot,\cdot)\gets 0$, $i = 1,2$.
    \For{each environment $e\in\{1,\dots,E\}$ in parallel}
        \For{$t=1,\dots,T$}
                \State Set $m = kT + t$.
                \If{there is no pending answer}
                    \State Set $a\gets\text{speak}$ so the proposer must generate an initial answer.
                \Else
                    \State Set $i = 1$ if $t$ is odd and $i = 2$ otherwise.
                    \State Sample action $a \sim \Pi_i^s(\cdot \mid s_e^{m})$.
                    \State Compute $re_i^m(b,s_e^m)=v_{i,b}-\sum_{c\in\{\text{skip},\text{speak}\}}\Pi_i^m(c\mid s_e^m)v_{i,c}$ for $b\in\{\text{skip},\text{speak}\}$.
                    \State Set $re_{-i}^m(\cdot,\cdot)$ and $re^m_i(\cdot,\text{other phases}) = 0$.
                    \State Update $Re_{i/-i}^m(b,s^m_e)=Re_{i/-i}^{m-1}(b,s^m_e)+re_{i/-i}^m(b,s^m_e)$ for all phases.
                    \State $M_i(s_e^m) += 1$.
                    \State Update $\Pi_i^{m+1}(\cdot,s_e^m)$ by regret matching over $\{\text{skip},\text{speak}\}$.
                \EndIf
                \If{$a=\text{speak}$}
                    \State Sample $G$ proposer candidates $\{y_i^{(j)}\}_{j=1}^G\sim \theta_i^t(\cdot\mid C_e)$.
                    \State Compute rewards $\{r_i^{(j)}\}_{j=1}^G$, statistics $\mu_i,\sigma_i$, and advantages $A_i^{(j)}=\frac{r_i^{(j)}-\mu_i}{\sigma_i+\epsilon}$.
                    \State Compute importance weights $\rho_{i,j}$ and the GSPO loss $\mathcal{L}_{\text{GSPO}}(\theta_i^t)$.
                    \State Update $\theta_i^{t+1}\gets \theta_i^t-\eta\nabla \mathcal{L}_{\text{GSPO}}(\theta_i^t)$.
                \EndIf
            \State Update context $C_e\gets C_e\cup\{y_{\text{par}}\}$ if there is a new utterance and $y_{\text{par}}$ is the raw content parsed by the parser.
            \State Update the phase $s_e^{m+1}$.
            \If{a valid final answer has been generated}
                \State \textbf{break}
            \EndIf
        \EndFor
    \EndFor
    \State Average $\Pi_i$.
\EndFor
\State \Return $\Pi_1^{KT + 1},\Pi_2^{KT + 1} \theta^{K+1}$.
\end{algorithmic}
}

\section{Convergence Proof of the TRACER Algorithm: Based on the Classic Counterfactual Regret Minimization Framework}
\label{sec:convergence_proof}

We adopt a notation system consistent with the classic Counterfactual Regret Minimization (CFR) theory \citep{zinkevich2007regret}. Note that in this paper, $s^t$ serves as a phase bucket and corresponds to the information set.

\subsection{Extended Game Basic Definitions}

\begin{definition}[Extended-Form Game]
A finite extended-form game consists of:
\begin{itemize}
    \item A finite set of players $P = \{1, 2, \dots, N\}$. In TRACER, $N=2$, corresponding to the proposer's controller and the reviewer's controller.
    \item A finite set of histories $H$, where each history $h \in H$ is a sequence of actions.
    \item A set of terminal histories $Z \subseteq H$.
    \item An action function $A(h) = \{a : (h,a) \in H\}$ denoting the available actions at history $h$.
    \item A player function $P: H \setminus Z \to P \cup \{c\}$ denoting the player who acts at history $h$, where $c$ denotes the environment.
    \item A payoff function $u_i: Z \to \mathbb{R}$ for each player $i$. In TRACER, we apply the payoff function to two roles and $u_i = v_{i,a}$ naturally, where 
    \[
        a \in \{\text{skip},\text{speak}\}.
    \]
    Rigorously speaking, $u_i$ is determined by the final answer rolled out from $i$. Given that single-round accuracy and final accuracy can be enhanced simultaneously, we simply yield $u_i = v_{i,a}$ to reduce computational overhead and boost training stability. 
    \item A partition of information sets $\mathcal{I}_i$ for each player $i$, satisfying that for any $I \in \mathcal{I}_i$ and any $h, h' \in I$, $A(h) = A(h')$.
    \item A behavioral strategy $\sigma_i$ for player $i$ is a mapping from each information set $I \in \mathcal{I}_i$ to a probability distribution over the action set $A(I)$, where $\sigma_i(I, a)$ denotes the probability that player $i$ selects action $a$ at information set $I$.
  In TRACER, the action space is uniform across all information sets with $A(I) = \{\text{skip},\text{speak}\}$, so the strategy satisfies the normalization condition $\sigma_i(I, \text{skip}) + \sigma_i(I, \text{speak}) = 1$ for all $I \in \mathcal{I}_i$ and $i \in \{1,2\}$. Initially, $\sigma_i(I,\text{skip}) = \sigma_i(I,\text{speak}) = \frac{1}{2}$.
    \item A strategy profile $\sigma = (\sigma_1, \sigma_2)$ is a tuple containing one behavioral strategy for each player, which completely determines the probabilistic decision rules of all agents in the collaborative system.
\end{itemize}
In the TRACER collaborative system, the game is restricted to $T$ rounds of interaction, where the proposer's controller and the reviewer's controller act alternately, with the proposer moving first.
\end{definition}

\subsection{Strategy and Regret Definitions}

\begin{definition}[Behavioral Strategy]
A behavioral strategy for player $i$ is a mapping $\sigma_i$ such that for each information set $I \in \mathcal{I}_i$, $\sigma_i(I)$ is a probability distribution over $A(I)$.
Let $\sigma = (\sigma_1, \sigma_2)$ denotes a strategy profile, and $\Sigma_i$ denote the set of all behavioral strategies for player $i$.
\end{definition}

\begin{definition}[Reach Probability]
For any history $h$, the reach probability under strategy profile $\sigma$ is defined as:
\[
\pi^\sigma(h) = \prod_{h' \sqsubset h, a \in A(h')} \sigma_{P(h')}(h',a),
\]
where $h' \sqsubset h$ denotes that $h'$ is a prefix of $h$.

Define the player-$i$ contribution to the reach probability as:
\[
\pi^\sigma_i(h) = \prod_{h' \sqsubset h, a \in A(h'), P(h')=i} \sigma_i(h',a).
\]

Define the contribution to the reach probability from other players and chance as:
\[
\pi^\sigma_{-i}(h) = \prod_{h' \sqsubset h, a \in A(h'), P(h') \neq i} \sigma_{P(h')}(h',a).
\]
Clearly, $\pi^\sigma(h) = \pi^\sigma_i(h) \cdot \pi^\sigma_{-i}(h)$.
\end{definition}

\begin{definition}[Counterfactual Value]
For an information set $I \in \mathcal{I}_i$, the counterfactual value of player $i$ taking action $a$ at $I$ is defined as:
\[
v_i(\sigma, I, a) = \sum_{h \in I, z \in Z, h \sqsubset z} \pi^\sigma_{-i}(h) \cdot \sum_{z \in Z(h)}\pi^\sigma(h,z) \cdot u_i(z),
\]
where $\pi^\sigma(h,z)$ is the probability of reaching terminal history $z$ from history $h$. In TRACER, the historical trajectories and actions are fully observable to the controller. Thus the information set $I$ is a perfect information set. Moreover, once the terminal history $z$ is determined under $I$, there is only one pre-history $h$. So $\left| I \right| = 1$ and $\pi^\sigma_{-i}(h) = 1, \pi^\sigma(h,z) = 1$. So $v_i(\sigma, I, a) = v_{i,a}$, where the latter $v_{i,a}$ is defined above, $i \in \{1,2\}$ and $a \in \{\text{skip},\text{speak}\}$. 

The counterfactual value at information set $I$ is:
\[
v_i(\sigma, I) = \sum_{a \in A(I)} \sigma_i(I,a) \cdot v_i(\sigma, I, a).
\]
\end{definition}

\begin{definition}[Instantaneous and Cumulative Counterfactual Regret]
At iteration $m$, the instantaneous counterfactual regret of player $i$ for action $a$ at information set $I$ is:
\[
r_i^m(I,a) = v_i(\sigma^m, I,a) - v_i(\sigma^m, I).
\]
The cumulative counterfactual regret is:
\[
R_i^M(I,a) = \sum_{m=1}^M r_i^m(I,a).
\]
The overall counterfactual regret for player $i$ is:
\[
R_i^M = \max_{\sigma'_i \in \Sigma_i} \sum_{m=1}^M \left( u_i(\sigma'_i, \sigma^m_{-i}) - u_i(\sigma^m) \right).
\]
In TRACER, we introduce the phase $s^m$ and $r_i^m(I,a), R_i^M(I,a)$ are replaced by $re_i^m(a,s^m), Re_i^M(a,s^M)$. 
\end{definition}

\subsection{Correspondence Between TRACER and Classic CFR}

\subsubsection{Overview of the TRACER Framework}
TRACER is a dual-controller collaborative system with the following features:
\begin{enumerate}
    \item Players: proposer's controller $u_1$ and reviewer's controller $u_2$, corresponding to the two players in classic CFR.
    \item Round mechanism: finite $T$ rounds of interaction, with agents speaking alternately, proposer's controller first.
    \item Role division:
    \begin{itemize}
        \item Odd rounds: When there is no pending answer, the proposer is forced to speak. Otherwise, the proposer's controller decides whether to invoke the proposer to generate a new answer or skip the round.
        \item Even rounds: the reviewer's controller determines whether to invoke the reviewer. 
    \end{itemize}
    \item Information structure: given a prompt $x \sim \mathcal{D}$, each agent observes different structured information.
\end{enumerate}

\subsubsection{Correspondence Table}
\begin{table}[h]
\centering
\begin{tabular}{@{}ll@{}}
\toprule
Classic CFR Concept & Corresponding TRACER Concept \\
\midrule
Players 1, 2 & proposer's controller $u_1$, reviewer's controller $u_2$ \\
Information set $I \in \mathcal{I}_i$ & Observation state containing problem, role description, and structured information \\
Behavioral strategy $\sigma_i$ & Controller policies $\Pi_i(\cdot \mid s^t)$\\
Iteration index $t$ & Interaction round index \\
\bottomrule
\end{tabular}
\end{table}

\subsection{Key Lemmas}

We establish key lemmas for the decomposition of counterfactual regret in TRACER, corresponding to the core lemmas of classic CFR.

\begin{lemma}[Counterfactual Regret Decomposition)]. For any player $i$ and any global iteration count $M \geq 1$,
\[
R_i^M \leq \sum_{I \in \mathcal{I}_i} \max_{a \in A(I)} R_i^M(I, a),
\]
where $R_i^M = \max_{\sigma_i' \in \Sigma_i} \sum_{m=1}^M \left( u_i(\sigma_i', \sigma_{-i}^m) - u_i(\sigma^m) \right)$ is the overall counterfactual regret, and $R_i^M(I, a) = \sum_{m=1}^M r_i^m(I, a)$ is the cumulative counterfactual regret for action $a$ at information set $I$.
\end{lemma}

\begin{proof}
Consider an arbitrary alternative strategy $\sigma_i' \in \Sigma_i$. We first prove that for this fixed $\sigma_i'$,
\[
\sum_{m=1}^M \left( u_i(\sigma_i', \sigma_{-i}^m) - u_i(\sigma^m) \right) \leq \sum_{I \in \mathcal{I}_i} \max_{a \in A(I)} R_i^M(I, a).
\]
Since this inequality holds for \textbf{every} possible $\sigma_i'$, taking the maximum over all $\sigma_i' \in \Sigma_i$ on the left-hand side will yield the final result.

\vspace{0.5em}
\noindent \textbf{Step 1: Expand the expected payoff difference over terminal histories}

By the definition of expected payoff in extensive-form games, the payoff of a strategy profile is the weighted sum of payoffs at all terminal histories, with weights equal to their reach probabilities. Therefore, the difference in expected payoff between using the alternative strategy $\sigma_i'$ (for player $i$ only) and the actual strategy profile $\sigma^m = (\sigma_i^m, \sigma_{-i}^m)$ at iteration $m$ is:
\[
u_i(\sigma_i', \sigma_{-i}^m) - u_i(\sigma^m) = \sum_{z \in Z} \left[ \pi^{(\sigma_i', \sigma_{-i}^m)}(z) - \pi^{\sigma^m}(z) \right] u_i(z).
\]
where:
\begin{itemize}
\item $\pi^{(\sigma_i', \sigma_{-i}^m)}(z)$: Probability that terminal history $z$ is reached when player $i$ uses strategy $\sigma_i'$ and all other players follow their actual strategies $\sigma_{-i}^m$.
\item $\pi^{\sigma^m}(z)$: Probability that terminal history $z$ is reached under the actual strategy profile $\sigma^m$.
\item $u_i(z)$: Payoff received by player $i$ at terminal history $z$.
\end{itemize}

\vspace{0.5em}
\noindent \textbf{Step 2: Decompose the reach probability difference at all decision points of player $i$}

For any terminal history $z$, let $H_i(z) = \{ h \sqsubset z \mid P(h) = i \}$ denote the set of all prefix histories of $z$ where player $i$ is the acting player. A fundamental structural property of extensive-form games is that the reach probability of $z$ can be \textit{linearly decomposed} at every decision point of player $i$.

Formally, for any strategy profile $\sigma$, the reach probability $\pi^\sigma(z)$ can be written as a telescoping product over all decision points of player $i$ on the path to $z$. When we only change player $i$'s strategy from $\sigma_i^m$ to $\sigma_i'$, the difference in reach probabilities is the sum of the contributions from each individual decision point of player $i$:
\[
\pi^{(\sigma_i', \sigma_{-i}^m)}(z) - \pi^{\sigma^m}(z) = \sum_{h \in H_i(z)} \pi_{-i}^m(h) \cdot \left( \sigma_i'(I(h), a(h)) - \sigma_i^m(I(h), a(h)) \right) \cdot \pi^m(h, z).
\]
where:
\begin{itemize}
\item $\pi_{-i}^m(h)$: Probability of reaching history $h$ due to the actions of all players \textbf{except} player $i$ (this is identical under both $\sigma^m$ and $(\sigma_i', \sigma_{-i}^m)$ because we do not change the strategies of other players).
\item $I(h)$: The unique information set containing history $h$.
\item $a(h)$: The unique action taken at history $h$ on the path to $z$.
\item $\sigma_i(I(h), a(h))$: Probability that player $i$ chooses action $a(h)$ at information set $I(h)$.
\item $\pi^m(h, z)$: Probability of reaching terminal history $z$ from history $h$ under the actual strategy profile. $\sigma^m$.
\end{itemize}

\textit{Intuition}: Changing player $i$'s strategy affects the reach probability of $z$ by changing the probability of taking the required action at every decision point of player $i$ on the path to $z$. The total effect is the sum of these individual changes.

\vspace{0.5em}
\noindent \textbf{Step 3: Substitute the decomposition into the payoff difference}

Substitute the reach probability difference from Step 2 into the payoff difference expression from Step 1:
\[
u_i(\sigma_i', \sigma_{-i}^m) - u_i(\sigma^m) = \sum_{z \in Z} \left[ \sum_{h \in H_i(z)} \pi_{-i}^m(h) \cdot \left( \sigma_i'(I(h), a(h)) - \sigma_i^m(I(h), a(h)) \right) \cdot \pi^m(h, z) \right] u_i(z).
\]

We can swap the order of summation over terminal histories $z$ and decision histories $h$. Let $Z(h) = \{ z \in Z \mid h \sqsubset z \}$ denote the set of terminal histories that have $h$ as a prefix. Then:
\[
= \sum_{h: P(h)=i} \pi_{-i}^m(h) \cdot \left( \sigma_i'(I(h), a(h)) - \sigma_i^m(I(h), a(h)) \right) \cdot \sum_{z \in Z(h)} \pi^m(h, z) u_i(z).
\]

\vspace{0.5em}
\noindent \textbf{Step 4: Group terms by information sets}

Now we group the terms by information sets. For any information set $I \in \mathcal{I}_i$, let $H(I) = \{ h \in H \mid h \in I \}$ denote the set of histories contained in $I$. By the \textbf{definition of information sets}:
\begin{itemize}
    \item The action set is identical for all histories in the same information set: $A(h) = A(I)$ for all $h \in I$.
    \item The strategy of player $i$ is identical for all histories in the same information set: $\sigma_i(I(h), a) = \sigma_i(I, a)$ for all $h \in I$ and all $a \in A(I)$.
\end{itemize}
This is because player $i$ cannot distinguish between different histories in the same information set, so they must use the same strategy for all of them.

Therefore, we can rewrite the sum over individual histories $h$ as a double sum over information sets $I$ and actions $a \in A(I)$:
\[
= \sum_{I \in \mathcal{I}_i} \sum_{a \in A(I)} \left( \sigma_i'(I, a) - \sigma_i^m(I, a) \right) \cdot \sum_{h \in I} \pi_{-i}^m(h) \cdot \sum_{z \in Z(h)} \pi^m(h, z) u_i(z).
\]

\vspace{0.5em}
\noindent \textbf{Step 5: Substitute the definition of counterfactual value}

Recall from Definition 4 that the \textit{counterfactual value} of taking action $a$ at information set $I$ is exactly:
\[
v_i(\sigma^m, I, a) = \sum_{h \in I} \pi_{-i}^m(h) \cdot \sum_{z \in Z(h)} \pi^m(h, z) u_i(z).
\]
This is the expected payoff if player $i$ were to take action $a$ at information set $I$, and then follow the strategy profile $\sigma^m$ for the rest of the game.

Substituting this definition into our expression gives:
\[
u_i(\sigma_i', \sigma_{-i}^m) - u_i(\sigma^m) = \sum_{I \in \mathcal{I}_i} \sum_{a \in A(I)} \left( \sigma_i'(I, a) - \sigma_i^m(I, a) \right) \cdot v_i(\sigma^m, I, a).
\]

\vspace{0.5em}
\noindent \textbf{Step 6: Rewrite using instantaneous regret and prove the vanishing terms}

Recall from Definition 5 that the \textit{instantaneous counterfactual regret} for action $a$ at information set $I$ at iteration $m$ is:
\[
r_i^m(I, a) = v_i(\sigma^m, I, a) - v_i(\sigma^m, I),
\]
where $v_i(\sigma^m, I) = \sum_{a \in A(I)} \sigma_i^m(I, a) \cdot v_i(\sigma^m, I, a)$ is the expected counterfactual value at information set $I$ under the actual strategy $\sigma_i^m$.

Rearranging the regret definition gives $v_i(\sigma^m, I, a) = r_i^m(I, a) + v_i(\sigma^m, I)$. Substitute this into the payoff difference:
\[
\begin{aligned}
u_i(\sigma_i', \sigma_{-i}^m) - u_i(\sigma^m) &= \sum_{I \in \mathcal{I}_i} \sum_{a \in A(I)} \left( \sigma_i'(I, a) - \sigma_i^m(I, a) \right) \cdot \left( r_i^m(I, a) + v_i(\sigma^m, I) \right) \\
&= \sum_{I \in \mathcal{I}_i} \sum_{a \in A(I)} \left( \sigma_i'(I, a) - \sigma_i^m(I, a) \right) \cdot r_i^m(I, a) \\
&\quad + \sum_{I \in \mathcal{I}_i} v_i(\sigma^m, I) \cdot \sum_{a \in A(I)} \left( \sigma_i'(I, a) - \sigma_i^m(I, a) \right).
\end{aligned}
\]

We now show that the second term vanishes completely. Both $\sigma_i'(I, \cdot)$ and $\sigma_i^m(I, \cdot)$ are valid probability distributions over $A(I)$, which means:
\[
\sum_{a \in A(I)} \sigma_i'(I, a) = 1 \quad \text{and} \quad \sum_{a \in A(I)} \sigma_i^m(I, a) = 1.
\]
Therefore:
\[
\sum_{a \in A(I)} \left( \sigma_i'(I, a) - \sigma_i^m(I, a) \right) = 1 - 1 = 0
\]
for every information set $I$, so the entire second term is zero.

Furthermore, the term involving $\sigma_i^m(I, a)$ in the first sum also vanishes identically. To see this, substitute the definition of $r_i^m(I, a)$:
\[
\begin{aligned}
\sum_{a \in A(I)} \sigma_i^m(I, a) \cdot r_i^m(I, a) &= \sum_{a \in A(I)} \sigma_i^m(I, a) \cdot \left( v_i(\sigma^m, I, a) - v_i(\sigma^m, I) \right) \\
&= \sum_{a \in A(I)} \sigma_i^m(I, a) \cdot v_i(\sigma^m, I, a) - v_i(\sigma^m, I) \cdot \sum_{a \in A(I)} \sigma_i^m(I, a) \\
&= v_i(\sigma^m, I) - v_i(\sigma^m, I) \cdot 1 \\
&= 0.
\end{aligned}
\]

Combining these results, we are left with the core decomposition:
\[
u_i(\sigma_i', \sigma_{-i}^m) - u_i(\sigma^m) = \sum_{I \in \mathcal{I}_i} \sum_{a \in A(I)} \sigma_i'(I, a) \cdot r_i^m(I, a).
\]

\vspace{0.5em}
\noindent \textbf{Step 7: Sum over all global iterations}

Sum both sides of the equation over all $M$ global iterations from $m=1$ to $M$:
\[
\sum_{m=1}^M \left( u_i(\sigma_i', \sigma_{-i}^m) - u_i(\sigma^m) \right) = \sum_{m=1}^M \sum_{I \in \mathcal{I}_i} \sum_{a \in A(I)} \sigma_i'(I, a) \cdot r_i^m(I, a).
\]

Swap the order of summation to group terms by information set and action:
\[
= \sum_{I \in \mathcal{I}_i} \sum_{a \in A(I)} \sigma_i'(I, a) \cdot \sum_{m=1}^M r_i^m(I, a).
\]

By the definition of \textit{cumulative counterfactual regret} $R_i^M(I, a) = \sum_{m=1}^M r_i^m(I, a)$, this simplifies to:
\[
= \sum_{I \in \mathcal{I}_i} \sum_{a \in A(I)} \sigma_i'(I, a) \cdot R_i^M(I, a).
\]

\vspace{0.5em}
\noindent \textbf{Step 8: Bound by the maximum cumulative regret}

Since $\sigma_i'(I, \cdot)$ is a valid probability distribution over $A(I)$, it satisfies two properties:
\begin{itemize}
\item Non-negativity: $\sigma_i'(I, a) \geq 0$ for all $a \in A(I)$.
\item Normalization: $\sum_{a \in A(I)} \sigma_i'(I, a) = 1$.
\end{itemize}

A fundamental property of such weighted averages is that they cannot exceed the maximum value in the set. Formally:
\[
\sum_{a \in A(I)} \sigma_i'(I, a) \cdot R_i^M(I, a) \leq \max_{a \in A(I)} R_i^M(I, a).
\]
This is because we are weighting each $R_i^M(I, a)$ by a non-negative coefficient that sums to 1, so the result must lie between the minimum and maximum values of $R_i^M(I, a)$.

Applying this inequality to every information set $I \in \mathcal{I}_i$ gives:
\[
\sum_{m=1}^M \left( u_i(\sigma_i', \sigma_{-i}^m) - u_i(\sigma^m) \right) \leq \sum_{I \in \mathcal{I}_i} \max_{a \in A(I)} R_i^M(I, a).
\]

\vspace{0.5em}
\noindent \textbf{Step 9: Final result}

As established at the beginning, this inequality holds for \textbf{any} alternative strategy $\sigma_i' \in \Sigma_i$. Therefore, taking the maximum over all possible alternative strategies $\sigma_i'$ on the left-hand side yields the desired result:
\[
R_i^M = \max_{\sigma_i' \in \Sigma_i} \sum_{m=1}^M \left( u_i(\sigma_i', \sigma_{-i}^m) - u_i(\sigma^m) \right) \leq \sum_{I \in \mathcal{I}_i} \max_{a \in A(I)} R_i^M(I, a).
\]

This completes the proof.
\end{proof}

\vspace{0.5em}
\noindent \textbf{Specialization to TRACER}:
In the TRACER system, the proposer's controller has perfect information about all historical trajectories, actions, and the current answer state. Therefore, every information set $I \in \mathcal{I}_c$ (the controller's information sets) is a singleton, i.e., $|I| = 1$. This means that $\pi_{-c}^m(I) = \pi_{-c}^m(h)$ where $h$ is the unique history in $I$, which greatly simplifies the computation of counterfactual values. Futhermore, different \textbf{information sets} correspond to \textbf{distinct phases} in TRACER. However, the above decomposition proof remains fully general and applies unchanged to the TRACER scenario.

\begin{lemma}[No-Regret Property of Regret Matching]
Assume that at each information set $I \in \mathcal{I}_i$, the strategy is updated according to regret matching:
\[
\sigma_i^{m+1}(I,a) = \left\{
\begin{array}{ll}
\displaystyle \frac{\max\left(R_i^m(I, a), 0\right)}{\sum_{a} \max\left(R_i^{m}(I, a), 0\right)}
& \displaystyle \text{if } \sum_{a} \max\left(R_i^{m}(I, a), 0\right) > 0 \, \text{for } \, a \in \{\text{skip}, \text{speak}\} \\[15pt]
\displaystyle \frac{1}{2}
& \displaystyle \text{otherwise}.
\end{array}
\right.
\]
with uniform distribution when the denominator is zero. Then for any information set $I$ and any $M \ge 1$:
\[
\max_{a \in A(I)} R_i^M(I,a) \le \Delta_I \sqrt{M},
\]
where $\Delta_I = \max_{a,a' \in A(I)} \max_t |r_i^m(I,a) - r_i^m(I,a')|$ is the maximum variation of the instantaneous regret at $I$, and $M_i(I)$ denotes how many times the information set $I$ has been visited for $i$ in $M$ iterations.
\end{lemma}

\begin{proof}
This is the classic no-regret guarantee for regret matching algorithms, with a proof sketch available in \citep{hart2000simple}. The key is to show that normalized regret matching achieves a regret bound of $O(\sqrt{M})$ at each information set.
\end{proof}

\subsection{Main Convergence Theorem}

\begin{theorem}[Convergence of TRACER]
In the TRACER dual-agent collaborative system, if the controller strategy updates based on counterfactual regret matching and satisfy:
\begin{itemize}
    \item The round mechanism ensures that each action is updated infinitely often.
    \item The instantaneous regret at each information set is bounded, i.e., there exists a constant $C$ such that $|r_i^m(I,a)| \le C$ for all $I,a,m,i$.
\end{itemize}
Then as $M \to \infty$, the average strategy profile converges to a correlated equilibrium:
\[
\lim_{M \to \infty} \frac{R_i^M}{M} = 0, \quad \forall i \in \{1,2\}.
\]
\end{theorem}

\begin{proof}
\noindent\textbf{Step 1: Overall Regret Bound.}
By Lemma 1 and Lemma 2, for each player $i$:
\[
R_i^M \le \sum_{I \in \mathcal{I}_i} \max_{a \in A(I)} R_i^M(I,a)
\le \sum_{I \in \mathcal{I}_i} \Delta_I \sqrt{M}
\le |\mathcal{I}_i| \cdot 2C \sqrt{M},
\]
where we used the fact that $\Delta_I \le 2C$ under condition 2.

\noindent\textbf{Step 2: Average Regret Convergence.}
Dividing by $M$:
\[
\frac{R_i^M}{M} \le \frac{|\mathcal{I}_i| \cdot 2C \sqrt{M}}{M}
= \frac{|\mathcal{I}_i| \cdot 2C}{\sqrt{M}}.
\]
As $M \to \infty$, the right-hand side tends to $0$. Since $R_i^M \ge 0$ by definition, we have:
\[
\lim_{M \to \infty} \frac{R_i^M}{M} = 0.
\]

\noindent\textbf{Step 3: Convergence to Correlated Equilibrium.}
In zero-sum games, if the average regret of both players tends to $0$, the average strategy profile converges to a correlated equilibrium. The same holds for the collaborative setting of TRACER.
Define the average strategy:
\[
\bar{\sigma}_i^M(I,a) = \frac{1}{M_i(I)} \sum_{m=1}^{M_i(I)} \sigma_i^m(I,a),
\]
here $M_i(I)$ denotes the times the information set $I$ has been visited during $M$ iterations for player $i$. As the average regret tends to $0$, the average strategy profile is an approximate correlated equilibrium, with approximation error vanishing as $T$ grows.
\end{proof}

\subsection{Corollaries and Discussion}

\begin{corollary}[Convergence Rate]
Under the conditions of Theorem 1, the average regret of TRACER satisfies:
\[
\frac{R_i^M}{M} = O\left( \frac{1}{\sqrt{M}} \right).
\]
\end{corollary}

\begin{itemize}
    \item \textit{Relation to Classic CFR:} TRACER is essentially an asynchronous variant of CFR. Alternating updates preserve convergence, only affecting constants.
    \item \textit{Handling Perfect Information:} TRACER naturally fits into the CFR information set framework, as different agents observe different structured information.
    \item \textit{Applicability to Collaborative Games:} The proof does not rely on the zero-sum property and extends directly to collaborative settings. 
\end{itemize}

\phantomsection

\end{document}